\documentclass[a4paper,11pt]{article}
\usepackage{fullpage} 

\usepackage{tikz}
\usetikzlibrary{trees}

\usepackage{diagbox} 
\usepackage{slashbox} 
\usepackage{tabularx} 
\usepackage{bbm} 
\usepackage{pifont}  

\usepackage{amsmath}
\usepackage{amssymb}
\usepackage{amsthm}
\usepackage{mathtools}
\usepackage{mathdots}
\usepackage{appendix}
\usepackage{graphicx}
\usepackage{enumerate}
\usepackage{microtype}
\usepackage{mleftright}
\usepackage{mdframed}
\usepackage{authblk}
\usepackage[most]{tcolorbox}
\usepackage[english]{babel}
\usepackage[utf8]{inputenc}
\usepackage{algorithm}
\usepackage[noend]{algpseudocode}
\usepackage{xspace}
\usepackage{tikz}
\usetikzlibrary{arrows.meta,positioning}
\usepackage{pgfplots}
\pgfplotsset{width=8cm,compat=1.9}
\usepackage{mathdots}
\usepackage{url}
\usepackage{forest}
\forestset{default preamble={for tree={s sep+=1cm}}}
\usepackage{caption,subcaption}
\usepackage[noadjust]{cite}
\usepackage{hyperref}

\newtheorem{theorem}{Theorem}[section]

\newtcolorbox{boxtheorembox}{colback=white,colframe=black,boxrule=0.8pt}

\newenvironment{boxtheorem}[1][]%
{\begin{boxtheorembox}\begin{theorem}[#1]}%
{\end{theorem}\end{boxtheorembox}}

\tcbset{
    rounded corners,
    colback = white,
    before skip = 0.2cm,
    after skip = 0.5cm,
    boxrule = 1.5pt,
    arc = 5pt
}

\makeatletter
\renewcommand{\ALG@name}{Protocol}
\makeatother
\allowdisplaybreaks

\graphicspath{ {./figures/} }

\usepackage{algorithm}
\usepackage[noend]{algpseudocode}

\DeclarePairedDelimiter{\ceil}{\lceil}{\rceil}
\DeclarePairedDelimiter{\floor}{\lfloor}{\rfloor}

\newtheorem*{theorem*}{Theorem}
\newtheorem{proposition}[theorem]{Proposition}
\newtheorem{lemma}[theorem]{Lemma}
\newtheorem{corollary}[theorem]{Corollary}
\newtheorem*{corollary*}{Corollary}

\theoremstyle{definition}
\newtheorem{definition}[theorem]{Definition}

\newcommand{\cH}{\mathcal{H}}
\newcommand{\cX}{\mathcal{X}}
\newcommand{\cY}{\mathcal{Y}}
\newcommand{\cS}{\mathcal{S}}
\newcommand{\cF}{\mathcal{F}}

\newcommand{\cB}{\mathcal{B}}

\newcommand{\cT}{\mathcal{T}}

\newcommand{\ERM}{\operatorname{ERM}}

\newcommand{\Lrn}{\mathsf{Lrn}}

\newcommand{\Maj}{\mathsf{Maj}}

\newcommand{\Ber}{\mathrm{Ber}}
\newcommand{\Bin}{\mathrm{Bin}}

\newcommand{\eps}{\varepsilon}

\newcommand{\etoe}{\mathsf{e2e}}
\newcommand{\ct}{\mathsf{CoT}}

\newcommand{\VC}{\mathtt{VC}}
\newcommand{\LD}{\mathtt{L}}

\newcommand{\Nat}{\mathtt{Nat}}

\newcommand{\ARL}{\mathrm{ATdim}}

\newcommand{\tree}{\mathbf{T}}

\newcommand{\ATdim}{\mathrm{ATdim}}
\newcommand{\lin}{\operatorname{lin}}
\newcommand{\Ldim}{\mathrm{Ldim}}

\newcommand{\Dim}{\mathrm{Dim}}

\newif\ifanonymous
\anonymousfalse

\begin{document}
\title{Sample Complexity of Autoregressive Reasoning: Chain-of-Thought vs.\ End-to-End}

\author[1]{Steve Hanneke}
\author[2]{Idan Mehalel}
\author[3]{Shay Moran}
\affil[1]{Purdue University}
\affil[2]{The Hebrew University}
\affil[3]{Technion and Google Research}


\maketitle

\begin{abstract}
Modern large language models generate text autoregressively, producing tokens one at a time. 
To study the learnability of such systems, Joshi et al.~(COLT 2025) introduced a PAC-learning framework for \emph{next-token generators}, the primitive underlying autoregressive models. 
In this framework, an unknown next-token generator maps a sequence of tokens to the next token and is iteratively applied for $T$ steps, producing a chain of tokens whose final token constitutes the model’s output. 
The learning task is to learn the input-output mapping induced by this autoregressive process.
Depending on the available supervision, training examples may reveal only the final output (\emph{End-to-End supervision}) or the entire generated chain (\emph{Chain-of-Thought supervision}). 
This raises two natural questions: how the sample complexity depends on the generation length $T$, and how much Chain-of-Thought supervision can reduce this dependence.

In this work we give a nearly complete answer to both questions by uncovering a \emph{taxonomy of how the sample complexity scales with $T$}.
For End-to-End learning, we show that the landscape is remarkably rich: subject to mild conditions, essentially any growth rate $r(T)$ between constant and linear can arise as the sample complexity, and combined with the linear upper bound of Joshi et al., this yields an essentially complete characterization.
In contrast, under Chain-of-Thought supervision we show that the sample complexity is \emph{independent of~$T$}, demonstrating that access to intermediate reasoning steps can eliminate the dependence on the generation length altogether. 
Our analysis introduces new combinatorial tools, and as corollaries we resolve several open questions posed by Joshi et al.~regarding the dependence of learnability on the generation length and the role of Chain-of-Thought supervision.
\end{abstract}




\section{Introduction} \label{sec:intro}





Imagine trying to learn mathematics from a master. One way is to observe only the final products: the statements they prove and the polished proofs they eventually present. Another is to also witness the process that led there: the intermediate ideas, the false starts, the examples they test, the approaches they abandon, and the partial structures they build along the way. In the first mode, one sees only the input-output behavior of expertise. In the second, one also sees the internal trajectory that connects the two. Intuitively, the latter seems far more informative. It reveals not just what the expert concluded, but how the conclusion was reached, and it is natural to expect that such richer supervision could substantially accelerate learning.

A closely related distinction has become central in modern machine learning, especially in the context of large language models. These models generate text autoregressively, producing one token at a time, and often solve complex tasks by generating intermediate reasoning steps before arriving at a final answer. This intermediate reasoning is commonly referred to as a \emph{chain of thought}. In practice, one may train or supervise such models in at least two different ways: either by providing supervision only on the final answer, or by also revealing the intermediate chain of thought that leads to it. The latter form of supervision has attracted considerable attention, both empirically and conceptually, as a possible explanation for improved reasoning performance.

Motivated by this distinction, Joshi et al.~\cite{joshi2025theory} introduced an elegant PAC-style framework for studying autoregressive generation in a clean and formal way. In their model, the basic object is a \emph{next-token generator}: a function that maps a partial sequence to the next token. Iterating such a generator produces a trajectory of intermediate tokens, which may be viewed as a chain of thought, and the final token is then interpreted as the end-to-end output. This framework makes it possible to ask, in precise statistical terms, how the sample complexity of learning depends on the length of the generated reasoning process, and to compare the power of end-to-end supervision with that of chain-of-thought supervision.

\subsection{The autoregressive learning model}

We now turn to the formal model introduced by Joshi et al.~\cite{joshi2025theory}. Fix an alphabet $\Sigma$, and let $\Sigma^\star$ denote the set of all finite strings over $\Sigma$. A \emph{next-token generator} is a function
\[
f:\Sigma^\star \to \Sigma,
\]
which maps each finite string $x\in\Sigma^\star$ to a single symbol $f(x)\in\Sigma$. Such a function naturally induces an autoregressive generation process. For every next-token generator $f$, define its \emph{apply-and-append} map
\[
\bar f:\Sigma^\star\to\Sigma^\star,
\qquad
\bar f(x)=x\circ f(x),
\]
where $x\circ f(x)$ is the string obtained by appending the symbol $f(x)$ to the end of $x$. Iterating this operation $T$ times yields the $T$-step continuation of the input. Writing
\[
\bar f^{\,T}
=
\underbrace{\bar f\circ\cdots\circ \bar f}_{T\text{ times}},
\]
we define the \emph{$T$-step chain of thought} of $f$ on $x$ to be the length-$T$ suffix of $\bar f^{\,T}(x)$, denoted
\[
f^{\mathsf{CoT},T}(x)\in\Sigma^T,
\]
and the corresponding \emph{end-to-end output} to be its last symbol,
\[
f^{\mathsf{e2e},T}(x)\in\Sigma.
\]
Thus, each base class $\mathcal F$ of next-token generators induces two derived classes,
\[
\mathcal F^{\mathsf{CoT},T}
:=
\{f^{\mathsf{CoT},T}:f\in\mathcal F\},
\qquad
\mathcal F^{\mathsf{e2e},T}
:=
\{f^{\mathsf{e2e},T}:f\in\mathcal F\}.
\]
The class $\mathcal F^{\mathsf{CoT},T}$ records the full $T$-step autoregressive trajectory, whereas $\mathcal F^{\mathsf{e2e},T}$ records only the final generated symbol.

We are interested in PAC learning the end-to-end prediction task induced by $\mathcal F$. Concretely, let $\mathcal D$ be an unknown distribution over $\Sigma^\star$, and let $f_\star\in\mathcal F$ be an unknown target generator. In both learning regimes, the learner receives i.i.d.\ \emph{labeled training examples} whose inputs are drawn from $\mathcal D$. The difference between the regimes lies only in the amount of label information revealed during training.

In the \emph{end-to-end} regime, each training example has the form
\[
\bigl(x, f_\star^{\mathsf{e2e},T}(x)\bigr),
\]
so the learner observes only the final generated symbol. Thus, this is exactly the standard PAC learning problem for the induced class $\mathcal F^{\mathsf{e2e},T}$.

In the stronger \emph{chain-of-thought supervision} regime, each training example has the form
\[
\bigl(x, f_\star^{\mathsf{CoT},T}(x)\bigr),
\]
so the learner observes the entire $T$-symbol autoregressive trajectory generated from $x$, and in particular also its final symbol.

At test time, however, the task is identical in both regimes: given a fresh unlabeled example~$x\sim\mathcal D$, the learner must predict the final symbol $f_\star^{\mathsf{e2e},T}(x)$. Accordingly, performance is always measured by the end-to-end error\footnote{Correctness in this model is defined with respect to a single target output. This is especially natural in the binary setting: when $\Sigma=\{0,1\}$, one may interpret input strings as prompts posing yes/no questions to an autoregressive expert, and such questions presumably have a unique correct binary answer.  For larger output spaces, however, there may naturally be several correct outputs -- for example, several valid proofs of the same theorem. In such settings, a more realistic formulation would use a more general loss function that assigns zero loss to multiple outputs, akin to multilabel prediction; we discuss this further in Section~\ref{sec:open-questions}.}
\[
\Pr_{x\sim\mathcal D}\bigl[h(x)\neq f_\star^{\mathsf{e2e},T}(x)\bigr].
\]

Our goal is to understand how the complexity of autoregressive learning scales with the generation length $T$. In particular, we ask how the sample complexity, namely the number of training examples needed to achieve a prescribed accuracy and confidence, depends on $T$ in the end-to-end and chain-of-thought supervision regimes, and to what extent chain-of-thought supervision can reduce this complexity.

\paragraph{Organization.}
Section~\ref{sec:main-results} presents our main results.
Sections~\ref{sec:overview}, \ref{sec:related}, and \ref{sec:prel} then provide a technical overview, discuss related work, and review the technical background, respectively.
The remaining sections contain the full proofs and additional results: Section~\ref{sec:e2e-proofs} treats the end-to-end setting, Section~\ref{sec:cot-proofs} treats learning with chain-of-thought supervision, and Section~\ref{sec:non-vc-proofs} contains additional results. 

\section{Main Results} \label{sec:main-results}

\paragraph{Summary of our results.}
We begin by summarizing the picture established by \cite{joshi2025theory}.
In the chain-of-thought (CoT) regime, they showed that the sample complexity grows at most logarithmically with the generation length $T$, whereas in the end-to-end (e2e) regime it grows at most linearly in $T$. In the e2e setting, they also exhibited classes for which this linear dependence is attained. In contrast, in the CoT setting, it remained open whether the logarithmic dependence on $T$ is ever necessary. They further proved that for classes of finite Littlestone dimension, the e2e sample complexity is at most logarithmic in $T$.
These results left several natural questions open: whether CoT truly requires any dependence on $T$ at all; whether intermediate rates between constant and linear are attainable in the e2e setting; and whether some dimension, perhaps refining Littlestone dimension, characterizes sublinear e2e behavior.

In this work, we resolve these questions as follows:

\begin{itemize}
    \item In the CoT regime, we show that for every base class of finite VC dimension, the sample complexity is independent of $T$. This closes the gap left by \cite{joshi2025theory} in the CoT setting by showing that the sample complexity is in fact independent of~$T$.

    \item In the e2e regime, we establish a taxonomy of achievable growth rates, showing that, subject to mild conditions, every rate between constant and linear can occur. This fills the main gap left open by \cite{joshi2025theory} concerning the attainability of intermediate rates, and in particular rules out a dichotomy between logarithmic and linear behavior.

    \item We use a diagonalization argument to prove that, in full generality, there is no dimension-based characterization of sublinear e2e rates. At the same time, we introduce a new tree-based combinatorial dimension that yields a sufficient condition for logarithmic growth. Thus, the question left open by \cite{joshi2025theory} has a negative answer in general, while in the logarithmic regime we refine and strengthen their Littlestone-dimension-based sufficient condition.
\end{itemize}

\paragraph{Roadmap for this section.}
In Section~\ref{sec:questions}, we introduce the formal questions and basic definitions. 
Section~\ref{sec:cot-intro} studies the chain-of-thought supervision regime, and Section~\ref{sec:e2e-intro} studies the end-to-end regime. 
Finally, Section~\ref{sec:open-questions} contains additional results and open questions.

\subsection{Questions and definitions} \label{sec:questions}
We focus on the binary case $\Sigma=\{0,1\}$. This corresponds to a setting in which the target autoregressive generator ultimately produces a binary answer after generating $T-1$ intermediate tokens. In particular, one may think of prompts as yes/no questions answered by an expert through an autoregressive reasoning process. Our results and techniques extend more generally to any finite token space $\Sigma$, but restricting to the binary setting keeps the presentation cleaner. The case of infinite $\Sigma$ lies beyond the scope of this paper.

Let us briefly recall the notion of sample complexity in PAC learning. For a domain~$\mathcal X$ and a hypothesis class $\mathcal H\subseteq\{0,1\}^{\mathcal X}$, let
\(m_{\mathcal H}(\epsilon,\delta)\)
denote the smallest number of i.i.d.\ labeled examples sufficient to ensure that, for every distribution on $\mathcal X$ and every target hypothesis in $\mathcal H$, the learner outputs with probability at least $1-\delta$ a predictor whose error on a fresh random example is at most $\epsilon$.

The key combinatorial parameter governing this quantity is the \emph{VC dimension}. Namely,~$\VC(\mathcal H)$ is the largest integer $d$ for which there exist points $x_1,\dots,x_d\in\mathcal X$ that are shattered by $\mathcal H$, meaning that every binary labeling of these points is realized by some hypothesis in $\mathcal H$. For classes of finite VC dimension, this parameter characterizes the sample complexity up to universal constant factors:
\[
m_{\mathcal H}(\epsilon,\delta)
=
\Theta\!\left(\frac{d+\log(1/\delta)}{\epsilon}\right).
\]
We refer to~\cite{vapnik:74,blumer1989learnability,hanneke2016optimal,shalev2014understanding} for proofs and further background. In particular, this characterization applies directly to the end-to-end problem, since it is simply the ordinary PAC learning problem for the induced class $\mathcal F^{\mathsf{e2e},T}$. Accordingly, in the end-to-end regime, the problem reduces to understanding how the VC dimension of the induced class $\mathcal F^{\mathsf{e2e},T}$ scales with the generation length $T$.

The chain-of-thought supervision regime is different. There, the learner is still tested only on its ability to predict the final token, but during training it receives additional information in the form of the full chain of thought. Thus, this is not formally a standard PAC classification problem: the training examples contain side information that is unavailable at test time. In this sense, chain-of-thought supervision breaks the usual symmetry between training and test examples, and its sample complexity is therefore not captured directly by the classical VC characterization.

This leads to the main questions studied in this work. Suppose that the base class $\mathcal F$ has finite VC dimension.\footnote{Throughout the paper, we restrict attention to base classes of finite VC dimension. This is a natural assumption, and in Section~\ref{sec:open-questions} we further justify it by showing that allowing arbitrary base classes can lead to highly pathological behavior. For example, there exist classes for which, for every even~$T$, the problem is trivially end-to-end learnable - indeed, the corresponding class has VC dimension $0$, so no training examples are needed at all - whereas for every odd~$T$ it is not learnable even under chain-of-thought supervision.}

\smallskip
\noindent
\emph{How does the sample complexity in the chain-of-thought supervision regime scale with the generation length $T$?}

\smallskip
\noindent
\emph{How does the sample complexity in the end-to-end regime scale with $T$?}

\smallskip
By comparing the answers in the two regimes, we aim to understand to what extent chain-of-thought supervision can reduce the sample complexity of autoregressive PAC learning.

\subsection{Chain-of-thought supervision} \label{sec:cot-intro}

We begin with the chain-of-thought supervision regime. Our first main result shows that, once the base class has finite VC dimension, chain-of-thought supervision eliminates any dependence on the generation length $T$.
\begin{boxtheorem}[CoT sample complexity is independent of $T$]\label{thm:cot-main-intro}
For every class $\mathcal F$ of binary next-token generators with $\VC(\mathcal F)<\infty$, there exists a constant $c=c(\mathcal F)$ such that for every $T\ge 1$ and $\eps,\delta>0$, the chain-of-thought supervision sample complexity satisfies
\[
m_{\mathcal F}^{\mathsf{CoT},T}(\epsilon,\delta)
\le
\frac{c}{\epsilon}
\left(
\log\frac{1}{\epsilon}
+
\log\frac{1}{\delta}
\right).
\]
\end{boxtheorem}
\noindent This improves over the previous bound of~\cite{joshi2025theory}, which incurred a logarithmic dependence on~$T$. Thus, for classes of finite VC dimension, the theorem yields a sample complexity upper bound in the chain-of-thought regime that is completely independent of the generation length.
The constant obtained by our proof is
\[
c = \tilde O(dd^\star),
\]
where $d$ and $d^\star$ denote the VC and dual VC dimensions of $\mathcal F$, and in particular $d^\star \le 2^{d+1}$~\cite{assouad1983densite}. 
For many natural classes, $d^\star$ is of the same order as $d$, for example for linear classifiers.

As for lower bounds, note that already in the case $T=1$, the problem reduces to ordinary PAC learning of the base class $\mathcal F$. Hence, the lower bound
\[
\Omega\!\left(\frac{d+\log(1/\delta)}{\epsilon}\right)
\]
follows from the classical VC characterization recalled in the previous subsection. Thus, two gaps remain between our upper bound and the optimal one suggested by the case $T=1$: our bound contains an additional factor of $\log(1/\epsilon)$, and the constant $c$ carries an extra dependence on the dual VC dimension. It would be interesting to determine whether the bound
\(
O\!\left(\frac{d+\log(1/\delta)}{\epsilon}\right)
\)
always holds in the chain-of-thought regime. In Section~\ref{sec:open-questions} we show that this is indeed the case for several natural families, including autoregressive classes induced by linear predictors.

\subsection{End-to-end learning} \label{sec:e2e-intro}
We now turn to the end-to-end regime. As explained above, this is an ordinary PAC classification problem for the induced class $\cF^{\etoe,T}$. Accordingly, its sample complexity is characterized, up to universal constant factors, by the VC dimension of $\cF^{\etoe,T}$. \emph{Thus, the main question in the end-to-end regime is to understand how $\VC(\cF^{\etoe,T})$ scales with the generation length $T$.}

\subsubsection{A taxonomy of possible growth rates} \label{sec:e2e-taxonomy-intro}
The work of~\cite{joshi2025theory} established the general upper bound
\[
\VC(\cF^{\etoe,T}) = O(\VC(\cF)\cdot T),
\]
and asked whether this linear dependence can be refined. Our first result gives a broad taxonomy of the possible growth rates. To state the result, we introduce a mild regularity condition on growth rates. We say that a function
\(r:\mathbb{N}_+\to\mathbb{N}_+\)
is a \emph{monotone-subadditive rate} if
\begin{enumerate}
    \item $r$ is monotone non-decreasing; that is, $r(T_1)\leq r(T_2)$ for all $T_1\leq T_2$.
    \item $r$ is subadditive; that is, $r(T_1 + T_2) \leq r(T_1) + r(T_2)$ for all $T_1, T_2 \in \mathbb{N}$.
\end{enumerate}
This class includes essentially all standard growth rates encountered in learning theory, such as~$T$, $\ceil*{T^c}$ for constants $0<c<1$, polylogarithmic rates, etc.

\begin{boxtheorem}[Taxonomy of end-to-end growth rates]\label{thm:e2e-taxonomy-intro}
For every monotone-subadditive rate $r$, there exists a class $\cF$ of binary next-token generators such that
\[
r(T) \leq \VC(\cF^{\etoe,T}) \leq r(T) + r(1) \leq 2r(T)
\qquad\text{for all }T\ge 1.
\]
\end{boxtheorem}
Thus, essentially all natural rates between $\Theta(1)$ and $\Theta(T)$ can arise in the end-to-end regime, revealing a rich taxonomy of possible dependences on the generation length~$T$. We stress that the notion of a monotone-subadditive rate is only a sufficient condition for the theorem, not a necessary one: the same conclusion may hold for more intricate growth rates as well. Nevertheless, the definition already covers essentially all standard growth rates encountered in learning theory.

\subsubsection{Can a dimension characterize sublinear rates?} \label{sec:dimension-question}
Our taxonomy result refines the picture developed in~\cite{joshi2025theory}. That work identified two benchmark regimes for the growth of $\VC(\cF^{\etoe,T})$: on the one hand, finite Littlestone dimension implies a logarithmic upper bound; on the other hand, there are classes for which the growth is linear in $T$. Moreover,~\cite{joshi2025theory} showed that linear growth is the largest possible. Theorem~\ref{thm:e2e-taxonomy-intro} complements this picture by showing that essentially every standard sublinear growth rate between these two extremes can occur.

This naturally raises a structural question suggested by~\cite{joshi2025theory}: is there a combinatorial dimension that distinguishes between base classes $\cF$ for which $\VC(\cF^{\etoe,T})$ grows sublinearly in $T$ and those for which it does not? Equivalently, can one hope for a dimension theory that captures exactly when the end-to-end sample complexity grows sublinearly with the generation length?

We address this question from two complementary directions. First, we show that in full generality there can be no dimension that characterizes sublinear rates. Second, we introduce below a natural tree-based dimension that yields a strictly broader sufficient condition for logarithmic growth than the Littlestone dimension. To formulate the first statement, we must specify what it means for a dimension to characterize sublinear behavior in a non-tautological way. Indeed, without some quantitative requirement, one could define a completely vacuous ``dimension'' by declaring
\[
\Dim(\cF)=
\begin{cases}
0 & \text{if }\VC(\cF^{\etoe,T})=o(T),\\
\infty & \text{otherwise}.
\end{cases}
\]
Such a tautological definition would separate sublinear from non-sublinear growth, but would provide no effective information about the actual rate when it is sublinear. In contrast, standard combinatorial dimensions in learning theory do not merely distinguish finite from infinite behavior; they also yield quantitative bounds. This motivates the following definition.

\begin{definition}[Characterizing sublinear rates] \label{def:sublinear-rates-intro}
Let $\Dim$ be any function that assigns to each base class $\cF$ either a natural number or $\infty$. We say that $\Dim$ \emph{characterizes sublinear rates} if there exists a function \(M:\mathbb{N}\times\mathbb{N}\to\mathbb{R}_{\ge 0}\) such that for every base class $\cF$, the following hold:
\begin{enumerate}
    \item $\Dim(\cF)<\infty$ if and only if the function $T\mapsto \VC(\cF^{\etoe,T})$ is sublinear in $T$;
    \item if $\Dim(\cF)<\infty$, then for every $T\ge 1$,
    \[
    \VC(\cF^{\etoe,T}) \le M(\Dim(\cF),T);
    \]
    \item for every fixed $d\in\mathbb{N}$, the function $T\mapsto M(d,T)$ is sublinear in $T$.
\end{enumerate}
\end{definition}

The second and third conditions ensure that whenever $\Dim(\cF)$ is finite, it comes with an effective sublinear upper bound on the growth of $\VC(\cF^{\etoe,T})$. In this sense, the definition rules out tautological dimensions and aligns with the role played by standard combinatorial dimensions in learning theory, such as VC and Littlestone dimension, which come with quantitative guarantees.

Our next result shows that no such dimension exists.

\begin{theorem}\label{thm:no-dimension-sublinear-intro}
There is no dimension that characterizes sublinear rates.
\end{theorem}
The proof combines Theorem~\ref{thm:e2e-taxonomy-intro} with a diagonalization argument; we discuss this further in the proof overview section.

\paragraph{A weaker sufficient condition for logarithmic growth.}

We now turn to a second structural result for end-to-end learning. While Theorem~\ref{thm:no-dimension-sublinear-intro} rules out a dimension-theoretic characterization of all sublinear rates, one may still hope for natural sufficient conditions implying specific rates. Our next result gives such a condition for logarithmic growth, thereby refining the Littlestone-based logarithmic bound of~\cite{joshi2025theory}.

The starting point is the following observation. Fix a base class $\cF$, a prompt $x\in\Sigma^\star$, and a generation length $T$. Applying each $f\in\cF$ to $x$ for $T$ steps yields a family of $T$-bit strings,
\[
\{f^{\mathsf{CoT},T}(x):f\in\cF\}\subseteq \{0,1\}^T.
\]
These strings induce a rooted binary tree of depth $T$: its vertices are all prefixes of strings in the above family, and edges correspond to extension by $0$ or by $1$. Intuitively, the richer this induced tree is, the more diverse the autoregressive behavior of the class can be on the prompt $x$.

To quantify this richness, we look for large perfect binary trees inside the induced prefix tree. Concretely, if $\tree_1$ and $\tree_2$ are rooted binary trees, we say that $\tree_1$ is a \emph{leveled subtree} of $\tree_2$ if there exists an injective map $\iota:V(\tree_1)\to V(\tree_2)$ such that:
\begin{enumerate}
    \item for every pair of vertices $u,v\in V(\tree_1)$, the vertex $u$ is a left descendant, respectively right descendant, of $v$ in $\tree_1$ if and only if $\iota(u)$ is a left descendant, respectively right descendant, of $\iota(v)$ in $\tree_2$; and
    \item $u$ and $v$ lie on the same level in $\tree_1$ if and only if $\iota(u)$ and $\iota(v)$ lie on the same level in $\tree_2$.
\end{enumerate}
See Figure~\ref{fig:leveled-subtree} for an illustration.

\begin{figure}[t]
\centering
\begin{tikzpicture}[
    scale=0.88,
    every node/.style={circle, draw, inner sep=0.7pt, minimum size=3.4mm},
    hl/.style={circle, draw, fill=black!18, inner sep=0.7pt, minimum size=3.4mm},
    edge from parent/.style={draw, line width=0.4pt}
]

\node[hl] (r) at (0,0) {};

\node (a) at (-2.0,-0.9) {};
\node (b) at ( 2.0,-0.9) {};
\draw (r)--(a);
\draw (r)--(b);

\node[hl] (c) at (-3.0,-1.8) {};
\node      (d) at (-1.0,-1.8) {};
\node[hl] (e) at ( 1.0,-1.8) {};
\node      (f) at ( 3.0,-1.8) {};
\draw (a)--(c);
\draw (a)--(d);
\draw (b)--(e);
\draw (b)--(f);

\node (g) at (-3.6,-2.7) {};
\node (h) at (-2.4,-2.7) {};
\node (i) at (-1.0,-2.7) {};
\node (j) at ( 0.4,-2.7) {};
\node (k) at ( 1.6,-2.7) {};
\node (l) at ( 3.0,-2.7) {};
\draw (c)--(g);
\draw (c)--(h);
\draw (d)--(i);
\draw (e)--(j);
\draw (e)--(k);
\draw (f)--(l);

\node[hl] (m) at (-3.6,-3.6) {};
\node[hl] (n) at (-2.4,-3.6) {};
\node      (o) at (-1.0,-3.6) {};
\node[hl] (p) at ( 0.4,-3.6) {};
\node[hl] (q) at ( 1.6,-3.6) {};
\node      (s) at ( 3.0,-3.6) {};
\draw (g)--(m);
\draw (h)--(n);
\draw (i)--(o);
\draw (j)--(p);
\draw (k)--(q);
\draw (l)--(s);

\end{tikzpicture}
\caption{A rooted binary tree with a highlighted perfect leveled binary subtree of depth $2$.}
\label{fig:leveled-subtree}
\end{figure}
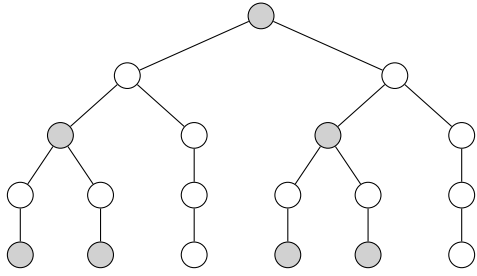

For a prompt $x\in\Sigma^\star$ and generation length $T$, let $G_{\cF,x,T}$ denote the induced prefix tree of the family \(\{f^{\mathsf{CoT},T}(x): f\in\cF\}.\)

\begin{definition}[Autoregressive tree dimension] \label{def:art-intro}
Let $\cF \subseteq \Sigma^{\Sigma^\star}$. The \emph{autoregressive tree dimension} of $\cF$, denoted $\ATdim(\cF)$, is the largest integer $d$ for which there exist a prompt $x \in \Sigma^\star$ and a generation length $T\ge 0$ such that $G_{\cF,x,T}$ contains a perfect leveled binary subtree of depth $d$. If such subtrees exist for arbitrarily large $d$, we define $\ATdim(\cF)=\infty$.
\end{definition}

This dimension is designed to capture the combinatorial expansion of the family of possible generations produced by the class on a fixed prompt. In particular, finite autoregressive tree dimension implies that the number of distinct $T$-step generations grows only polynomially with~$T$, and this in turn yields the following logarithmic upper bound.

\begin{theorem}[Autoregressive tree dimension implies logarithmic growth]\label{thm:artdim-log-intro}
For every base class~$\cF$ of binary next-token generators,
\[
\VC(\cF^{\etoe,T})
=
O\!\bigl(\ATdim(\cF)\cdot \VC(\cF)\cdot \log T\bigr).
\]
\end{theorem}

The autoregressive tree dimension is always at most the Littlestone dimension, and can be arbitrarily smaller. Moreover, it may be finite even when the Littlestone dimension is infinite. Thus, Theorem~\ref{thm:artdim-log-intro} yields a strictly broader sufficient condition for logarithmic growth than the one obtained in~\cite{joshi2025theory}: there are classes with infinite Littlestone dimension but finite autoregressive tree dimension, and hence logarithmic end-to-end growth is not implied by the earlier result but does follow from ours. See Section~\ref{sec:art} for further details and comparisons with classical dimensions.

\subsection{Open questions and additional results} \label{sec:open-questions}
We begin with the chain-of-thought supervision regime. Our main bound in Theorem~\ref{thm:cot-main-intro} is independent of the generation length $T$, but it is unlikely to be optimal in its dependence on the accuracy parameter $\epsilon$ and on the complexity of the base class. In particular, our proof incurs an additional multiplicative dependence on the dual VC dimension, which in general may be exponentially larger than the VC dimension. This leads to the following natural question: if $\VC(\cF)=d<\infty$, is it always possible to obtain the optimal bound
\(O\!\left(\frac{d+\log(1/\delta)}{\epsilon}\right)\) for chain-of-thought supervision?

Our next result confirms this for a broad family of natural classes, namely those admitting a stable sample compression scheme. We briefly recall the relevant notion from~\cite{bousquet2020proper}. Let $\mathcal H\subseteq\{0,1\}^{\mathcal X}$ be a concept class. A learning rule $A$ is called a \emph{stable sample compression scheme of size $k$} for $\mathcal H$ if, for every realizable\footnote{That is, a labeled dataset $S=((x_i,y_i))_{i=1}^m \in (\mathcal X\times\{0,1\})^m$ that is consistent with some hypothesis $h\in\mathcal H$, namely $h(x_i)=y_i$ for all $i\in[m]$.} dataset $S$, there exists a subdataset $K\subseteq S$ of size at most~$k$ such that:
\begin{enumerate}
    \item for every dataset $S'$ satisfying $K\subseteq S'\subseteq S$, one has $A(S')=A(S)$;
    \item the hypothesis $A(S)$ is consistent with~$S$.
\end{enumerate}
In words, the output of the learning rule is determined by a small subdataset $K$, and this dependence is stable under removal of examples outside $K$. Stable compression schemes arise in many natural settings, and in particular in learning algorithms based on linear/convex optimization. A basic example is maximum-margin linear classification in $\mathbb{R}^d$, where the output classifier is determined by at most $d+1$ support vectors, see Figure~\ref{fig:svm-stable-compression}.
\begin{figure}[t]
\centering
\begin{tikzpicture}[scale=1.0]

\draw[thick] (0,-2.2) -- (0,2.2);

\draw[gray] (-0.45,0) -- (0.45,0);
\draw[gray] (0.45,0) -- (0.45,0.55);
\draw[gray] (0.45,0) -- (0.45,-0.55);

\filldraw[red!80!black] (-3.2,1.3) circle (2.2pt);
\filldraw[red!80!black] (-2.4,-1.5) circle (2.2pt);
\filldraw[red!80!black] (-2.7,0.1) circle (2.2pt);
\filldraw[red!80!black] (-1.9,1.9) circle (2.2pt);
\filldraw[red!80!black] (-0.45,0.0) circle (2.2pt); 

\filldraw[blue!80!black] (0.45,0.55) circle (2.2pt); 
\filldraw[blue!80!black] (0.45,-0.55) circle (2.2pt); 
\filldraw[blue!80!black] (1.6,1.5) circle (2.2pt);
\filldraw[blue!80!black] (2.1,0.0) circle (2.2pt);
\filldraw[blue!80!black] (1.8,-1.0) circle (2.2pt);
\filldraw[blue!80!black] (3.0,-1.6) circle (2.2pt);

\draw[line width=0.5pt] (-0.45,0) circle (0.18);
\draw[line width=0.5pt] (0.45,0.55) circle (0.18);
\draw[line width=0.5pt] (0.45,-0.55) circle (0.18);

\end{tikzpicture}
\caption{A geometric illustration of stable compression for maximum-margin linear classification. In $\mathbb{R}^d$, the output classifier is determined by the (at most) $d+1$ of support vectors (circled), while removing the remaining sample points does not change the separator.}
\label{fig:svm-stable-compression}
\end{figure}
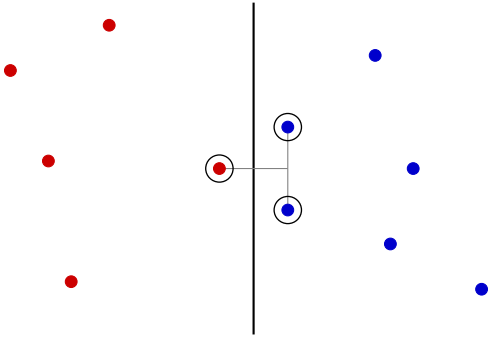
This notion yields the following consequence for chain-of-thought learning.

\begin{proposition}[Stable compression and CoT sample complexity]\label{prop:stable-compression-cot-intro}
Let $\cF$ be a base class of binary next-token generators, and suppose that $\cF$ admits a stable sample compression scheme of size $k$. Then for every $T\ge 1$, the chain-of-thought supervision sample complexity satisfies
\[
m_{\cF}^{\mathsf{CoT},T}(\epsilon,\delta)
=
O\!\left(\frac{k+\log(1/\delta)}{\epsilon}\right).
\]
\end{proposition}

Thus, whenever the base class admits stable compression of size comparable to its VC dimension, chain-of-thought supervision achieves the optimal dependence. Moreover, ~\cite{bousquet2020proper} showed that every class with finite Littlestone dimension admits a stable sample compression scheme of size $\Ldim(\cF)$, and Proposition~\ref{prop:stable-compression-cot-intro} therefore yields the corresponding bound with compression size $\Ldim(\cF)$.

The proposition also applies to natural geometric autoregressive classes. One concrete example is the class $\cF_{d,\lin}$ of linear autoregressors studied by~\cite{joshi2025theory}, defined by
\[
f_{w,\theta}(x)
=
\mathbf{1}\!\left[\langle w,\mathrm{tail}_d(x)\rangle \ge \theta\right],
\]
where $\mathrm{tail}_d(x)\in\{0,1\}^d$ denotes the suffix of $x$ of length at most $d$, padded with zeros on the left when $|x|<d$. For this class, maximum-margin linear classification yields a stable compression scheme of size $d+1$, and therefore Proposition~\ref{prop:stable-compression-cot-intro} gives
\[
m_{\cF_{d,\lin}}^{\mathsf{CoT},T}(\epsilon,\delta)
=
O\!\left(\frac{d+\log(1/\delta)}{\epsilon}\right),
\]
which is optimal up to universal constants. This improves the bound of~\cite{joshi2025theory}, which had quadratic dependence on $d$.

\medskip

A second direction for future research is to move beyond the restrictive finite-alphabet, single-answer setting considered in this paper. On the one hand, it is natural to consider larger, possibly infinite alphabets $\Sigma$, as well as more general loss functions that allow several outputs to be counted as correct. This is particularly relevant in settings such as theorem proving, where a prompt may encode a theorem statement and there may be many valid proofs. On the other hand, it is natural to ask whether one can relax the assumption that the base class itself is learnable. Indeed, it may happen that predicting the very next symbol is difficult, while producing a correct final answer after a longer autoregressive computation is much easier.

Both directions raise substantial new challenges. Our current analysis relies on ingredients such as uniform convergence and VC theory, which no longer apply to unbounded alphabets. Likewise, once one allows multiple correct outputs, the single-correct label classification viewpoint breaks down and should be replaced by a richer loss-based framework. At the same time, although it is natural to relax assumptions such as learnability of the base class, some replacement restrictions will likely be necessary, since otherwise pathological behavior can occur and obstruct a meaningful general theory.

The next result illustrates this point.

\begin{proposition}\label{prop:pathological-parity-intro}
There exists a base class $\cF$ such that for every even $T$, one has $\VC(\cF^{\etoe,T})=0$, and hence the class $\cF^{\etoe,T}$ is trivially PAC learnable, whereas for every odd $T$, the problem is not even learnable under chain-of-thought supervision.
\end{proposition}

Thus, without additional restrictions on the base class, the dependence on the generation length $T$ can be highly irregular. In particular, even the basic question of learnability may oscillate with the parity of $T$. This shows that assumptions such as finite VC dimension are not merely technical conveniences, but are needed to obtain a meaningful structural theory.

It is also worth noting that our proof of Theorem~\ref{thm:cot-main-intro} yields a stronger guarantee than the theorem statement itself suggests. Namely, the same argument in fact learns the full chain-of-thought class $\cF^{\ct,T}$, rather than only the final output under $\ct$-supervision. In this sense, our positive result applies even to the more demanding task of recovering the entire trajectory.

At the same time, we show in Section~\ref{sec:cot-taxonomy} that for every base-class $\cF$, if $\VC(\cF)=\infty$, then the class $\cF^{\ct,T}$ is not PAC learnable. Thus, any approach that attempts to exploit chain-of-thought supervision by explicitly learning the entire trajectory can succeed only in the finite-VC setting. Our theorem resolves this case positively, while the preceding proposition shows that beyond it one cannot expect a similarly clean general theory without additional assumptions.

Another natural extension is to allow the final output itself to be a string rather than a single symbol. This would better reflect realistic autoregressive tasks, in which a prompt is followed by an intermediate chain of thought and then by a nontrivial output sequence. Extending the theory to this more general prompt--reasoning--output format is a natural direction for future work.

\section{Technical Overview} \label{sec:overview}
The goal of this section is to give a high-level overview of the proofs in the paper. We highlight the main challenges behind the different results, the key ideas used to overcome them, and the overall structure of the arguments.


\subsection{Proof overview of Theorem~\ref{thm:cot-main-intro}: CoT sample complexity}

The previous result of Joshi et al.~\cite{joshi2025theory} established a sample complexity upper bound with an additional logarithmic dependence on the generation length $T$. Their starting point, which we also adopt, is to view chain-of-thought supervision as the task of learning the \emph{entire} generated trajectory, rather than only the final output token. Thus, the problem becomes a multiclass learning problem in which each label is a binary string in $\{0,1\}^T$, and hence the number of possible labels grows as $2^T$.

The learning rule considered in~\cite{joshi2025theory} is the most basic consistent rule: given the training sample, output any hypothesis whose generated trajectories agree with all observed trajectories. To analyze this learner, they study the corresponding loss class and bound its growth function by reducing it to the behavior of the base class $\cF$ on all prefixes that may arise along the observed trajectories. Since a sample of size $m$ contributes about $mT$ such prefixes, Sauer's lemma yields a bound polynomial in $mT$, and standard uniform-convergence arguments then lead to a sample complexity upper bound with an extra $\log T$ factor. Thus, the logarithmic dependence on $T$ is closely tied to the use of a generic consistent learner together with a covering argument over all possible behaviors on these $mT$ prefixes.

To remove this dependence on $T$, we depart from the generic consistent-learning approach and instead design a more structured learning rule based on \emph{sample compression}. Roughly speaking, a sample compression scheme consists of a compression map, which selects a small labeled subsample together with a small amount of side information, and a reconstruction map, which recovers from this compressed representation a hypothesis consistent with the original sample. General results in learning theory show that compression schemes of small size imply correspondingly strong sample complexity bounds. See also the support-vector example in Figure~\ref{fig:svm-stable-compression}, which illustrates this paradigm in a familiar geometric setting.

The key observation behind our construction is that each chain-of-thought example can be \emph{inflated} into $T$ ordinary binary-labeled examples for the base class. Concretely, if $(x,y) \in \Sigma^\star \times \Sigma^T$ is labeled by some $f_\star^{\ct,T}\in \cF^{\ct,T}$, where $y=(y_1,\dots,y_T)$, then it gives rise to the examples
\[
(x\, y_{\le i-1}, y_i), \qquad i\in[T],
\]
all of which are consistent with the underlying next-token predictor $f_\star \in \cF$. Thus, from a sample labeled by $\cF^{\ct,T}$ we obtain an inflated binary sample labeled by $\cF$. This reduction allows us to bring into play the majority-vote-based sample compression algorithm of~\cite{moran2016sample}.

At a high level, we run that compression procedure on the inflated sample, and then translate the resulting compressed representation back into a compression of the original chain-of-thought sample. The main subtlety is that the compression algorithm of~\cite{moran2016sample} reconstructs its hypothesis as a majority vote of a small collection of weak learners. In our setting, in order to lift the argument back from the inflated sample to the original chain-of-thought sample, these weak learners cannot be chosen arbitrarily: each of them must arise as an ERM for the base class on a suitable subcollection of inflated examples induced by a small number of original CoT examples. This is crucial for ensuring that every weak learner can itself be encoded using only a small number of examples from the original sample, rather than from the inflated sample. The fact that adding further compatible inflated examples does not hurt the weak-learning guarantee is what makes this modification possible. This ERM-based step is a key ingredient in the proof, and is also the point at which the argument uses in an essential way the assumption that the token space $\Sigma$ is bounded. In particular, the same idea extends beyond the binary case to any bounded $\Sigma$, but does not seem to extend to unbounded token spaces.


\subsection{Proof overview of Theorem~\ref{thm:e2e-taxonomy-intro}: Taxonomy of End-to-End growth rates}

The proof proceeds in two steps. We first treat the case of monotone-subadditive rate functions~$r$ satisfying~$r(1)=1$. Here the goal is to construct a base class of VC dimension $1$ whose end-to-end VC growth realizes the prescribed rate $r(T)$. We then extend the result to general monotone-subadditive rates by combining several copies of this construction in a way that scales the entire growth function by an additive factor of $r(1)$.

\paragraph{The case $\boldsymbol{r(1)=1}$.}

We begin with a basic VC-dimension-one family that already captures the maximal possible linear growth. For every infinite binary sequence \(a\in\{0,1\}^{\mathbb N}\), define a function \(f_a:\{0,1\}^\star\to\{0,1\}\) by
\[
f_a(x):=
\begin{cases}
a_{|x|+1} & \text{if } x \text{ is a prefix of } a,\\
0 & \text{otherwise.}
\end{cases}
\]
Let
\[
\cF_{\mathrm{full}}:=\{f_a : a\in\{0,1\}^{\mathbb N}\}.
\]
Thus, every function in \(\cF_{\mathrm{full}}\) is parametrized by an infinite binary sequence, and along the unique prefix chain of that sequence it reveals the next bit, while off that chain it outputs \(0\). This class has VC dimension \(1\), and moreover
\[
\VC(\cF_{\mathrm{full}}^{\etoe,T}) = T
\qquad\text{for every }T\ge 1.
\]
The lower bound is witnessed by the chain
\[
0,0^2,\dots,0^T,
\]
which can be labeled arbitrarily after \(T\) autoregressive steps. For example, the all-$1$ labeling is realized by any sequence beginning with \(0^T1^T\), while the all-$0$ labeling is realized by the sequence \(0^{2T}\). More generally, for any prescribed labeling of \(0,0^2,\dots,0^T\), one can choose the bits in positions \(T+1,\dots,2T\) accordingly. 

For the upper bound, observe first that any set shattered by \(\cF_{\mathrm{full}}^{\etoe,T}\) must be totally ordered by the prefix relation: if two strings are incomparable, then no function in \(\cF_{\mathrm{full}}\) can label both of them by \(1\). Thus every shattered set has the form
\[
x_1 \prec x_2 \prec \cdots \prec x_m .
\]
The key additional point is that all strings in such a shattered chain must have lengths contained in an interval of size at most \(T\). Indeed, if
\[
|x_m|-|x_1| \ge T,
\]
then the value of \(f^{\etoe,T}(x_1)\) is already determined by the bit of the underlying sequence at position \(|x_1|+T\), and this bit is already encoded in the longer prefix \(x_m\). Consequently, the labels of \(x_1\) and \(x_m\) cannot be chosen independently, contradicting shattering. Therefore every shattered chain satisfies \(m\le T\), and hence
\[
\VC(\cF_{\mathrm{full}}^{\etoe,T})=T.
\]

To realize an arbitrary admissible rate \(r\) with \(r(1)=1\), we now pass to a suitable subclass of this family. The proof is organized in two steps.

First, given any set \(N\subseteq\mathbb N\), we consider the subclass obtained by taking indicators of shifted subsets of \(N\): for every \(A\subseteq N\) and shift \(s\in\mathbb N\), let \(b^{\,s+A}\) be the indicator of the set \(s+A=\{s+a: a\in A\}\), and let \(f_{s+A}\) be the corresponding autoregressive predictor. Let \(\cF(N)\) be the resulting class. Since \(\cF(N)\subseteq \cF_{\mathrm{full}}\), the base class still has VC dimension at most \(1\). The key point is that the \(T\)-iterated VC dimension of \(\cF(N)\) is governed exactly by the one-dimensional density profile of \(N\):
\[
\VC(\cF(N)^{\etoe,T})
=
\max_{u\in\mathbb N} |N\cap [u+1,u+T]|.
\]
The lower bound is witnessed by an interval \([u+1,u+T]\) where the maximum is attained. If
\[
u+t_1,\dots,u+t_m
\]
are the points of \(N\) in this interval, then the all-zero chain
\[
0^{t_1}\prec \cdots \prec 0^{t_m}
\]
can be labeled arbitrarily after \(T\) autoregressive steps by choosing an appropriate subset of those points. For the upper bound, we use the same prefix-chain argument as above: any shattered set must lie on a single prefix chain. Moreover, the relevant output positions \(|x|+T\) must all correspond to active locations of a single shifted copy \(s+A\). The same comparison between the shortest and longest strings in the chain shows that these locations lie in an interval of length \(T\), and hence the chain can involve at most as many points as \(N\) has in such an interval.

Thus the problem reduces to a purely combinatorial question: which functions can be written as
\[
T\mapsto \max_{u\in\mathbb N} |N\cap [u+1,u+T]|?
\]
The answer is exactly the class of monotone-subadditive rates with \(r(1)=1\). Indeed, every such interval-density function is monotone and subadditive, and satisfies \(r(1)=1\) whenever \(N\neq\emptyset\). Conversely, if \(r\) is monotone, subadditive, and \(r(1)=1\), then \(r\) increases by at most one at each step. Let \(t_k\) be the first time at which \(r\) reaches the value \(k\), and set
\[
N:=\{t_k : k\ge 1\}.
\]
Then \(|N\cap[1,T]|=r(T)\), and subadditivity implies that no interval of length \(T\) contains more than \(r(T)\) points of \(N\). Hence
\[
\max_{u\in\mathbb N} |N\cap [u+1,u+T]| = r(T).
\]

Combining the two steps yields a VC-dimension-one base class \(\cF\) satisfying
\[
\VC(\cF^{\etoe,T})=r(T)
\qquad\text{for all }T.
\]
\paragraph{The extension to general $\boldsymbol{r(1)}$.}

To pass from the case $r(1)=1$ to arbitrary monotone-subadditive rates, we use a direct-product idea. At an abstract level, given concept classes $\cH_1,\ldots,\cH_k$ on pairwise disjoint domains, their Cartesian product is the class whose concepts are tuples $(h_1,\ldots,h_k)$ with $h_i\in\cH_i$, defined over the disjoint union of the underlying domains, by letting each component act on its own part of the domain. Concretely, if $x$ lies in the domain of $\cH_i$, then
\[
(h_1,\ldots,h_k)(x)=h_i(x).
\] A standard and straightforward observation is that VC dimension is additive under this product:
\[
\VC(\cH_1\times\cdots\times \cH_k)=\VC(\cH_1)+\cdots+\VC(\cH_k).
\]

We apply this idea with $k=r(1)$ copies of a class realizing the normalized rate
\[
\tilde r(T):= \ceil*{\frac{r(T)}{r(1)}},
\]
which satisfies $\tilde r(1)=1$. To use the first part of the proof, we need to prove that $\tilde{r}(T)$ is a monotone-subadditive rate. This is true by the order-preserving properties of the ceiling function. Now, by the first part of the proof, there exists a class $\cF$ with $\VC(\cF^{\etoe,T})=\tilde r(T)$.
Taking a Cartesian product of $r(1)$ copies of $\cF$ then produces a class of VC dimension $r(1)$ and end-to-end growth $\Theta(r(T))$.

The final point is that this abstract Cartesian-product operation can be simulated within the autoregressive framework. To do so, we assign the different copies of the class to disjoint regions of the input space using a prefix-free code, for example the prefixes
\[
1,01,001,0001,\ldots
\]
That is, the first copy acts only on strings beginning with $1$, the second on strings beginning with $01$, and so on. Because these prefixes are prefix-free, the autoregressive evolution associated with one copy never enters the region of another. Hence the combined class behaves exactly like a Cartesian product of the component classes, and the same holds for the corresponding end-to-end classes.

\subsection{Proof overview of Theorem~\ref{thm:no-dimension-sublinear-intro}: No characterization of sublinear rates}

We now turn to the negative result showing that there is no dimension that characterizes sublinear end-to-end growth in a quantitative way. More precisely, recall that such a characterization would require a dimension $\Dim$ together with a sublinear function $M(d,T)$ such that $\Dim(\cF)<\infty$ if and only if $\VC(\cF^{\etoe,T})$ grows sublinearly in $T$, and moreover whenever $\Dim(\cF)=d$ one has the explicit upper bound
\[
\VC(\cF^{\etoe,T}) \leq M(d,T).
\]
The point of the theorem is that no such dimension can exist.

At an intuitive level, the obstruction comes directly from Theorem~\ref{thm:e2e-taxonomy-intro}. That theorem shows that even among classes with VC dimension $1$, the possible sublinear growth rates of $\VC(\cF^{\etoe,T})$ are extremely rich: in particular, one can realize rates that are still sublinear, but arbitrarily close to linear. Thus, any putative dimension-based upper bound would have to dominate \emph{all} sublinear rates arising from finite-dimensional classes. The proof shows that this is impossible.

The formal proof proceeds by a diagonalization argument in the spirit of similar impossibility results on dimensions due to Lechner and Ben-David~\cite{LechnerB24}. A useful toy analogue is the following elementary puzzle from analysis. Suppose we are given a countable family of sequences
\[
(a_n^1), (a_n^2), (a_n^3), \ldots
\]
Can such a family be universal in the sense that for every sequence $b_n$, one of the sequences $a_n^i$ eventually dominates it? That is, can it happen that for every $b_n$ there exists $i$ such that
\[
a_n^i > b_n
\qquad\text{for all sufficiently large } n?
\]
The answer is no. Indeed, given the family $(a_n^i)$, define
\[
b_n :=\max\{a_m^i : i,m \le n\}.
\]
Then, for every fixed $i$ there are infinitely many $n$ for which
\[
b_n \ge a_n^i,
\]
so no sequence in the family eventually dominates $b_n$.

Our proof follows exactly this logic. If such a dimension $\Dim$ existed, then for each value $d$ we would obtain a sublinear rate $M(d,T)$, and hence a countable family of sublinear functions indexed by $d\in\mathbb N$, which would have to eventually dominate every sublinear rate realized by a finite-dimensional class. But Theorem~\ref{thm:e2e-taxonomy-intro} allows us to construct finite-VC classes realizing essentially arbitrary monotone-subadditive sublinear rates. Applying the same diagonal idea, we build a sublinear rate that eventually exceeds every candidate bound $M(d,T)$, while still remaining sublinear. The taxonomy theorem then provides a class realizing this diagonal rate, contradicting the assumed universality of the dimension-based bound.

In fact, the same argument yields a more general impossibility statement: for any at-most-linear benchmark $f(T)$, which satisfies $\lim_{T\to\infty}f(T)=\infty$, no dimension can in general separate classes with $\VC(\cF^{\etoe,T}) = o(f(T))$ from those with $\VC(\cF^{\etoe,T}) = \Omega(f(T))$ via a quantitative upper bound depending only on the dimension. We state and prove the theorem only for the benchmark $f(T)=T$, since sublinear growth was already the main focus of the previous work of~\cite{joshi2025theory}, making it a particularly natural case to study.

\subsection{Proof overview of Theorem~\ref{thm:artdim-log-intro}: A sufficient condition for logarithmic growth}

We now turn to the positive result showing that finiteness of our refined tree-based dimension implies logarithmic end-to-end growth. The theorem shows that if $\ATdim(\cF)$ is finite, then the quantity $\VC(\cF^{\etoe,T})$ grows only logarithmically with $T$, with dependence also on $\ATdim(\cF)$ and on $\VC(\cF)$.

The proof follows the same general strategy as the proof of the linear-in-$T$ upper bound of \cite{joshi2025theory}, but replaces the crude counting argument used there by a more refined combinatorial analysis of the relevant generation trees.

Suppose that $\{x_1,\dots,x_m\}$ is shattered by $\cF^{\etoe,T}$. Then, by definition, the class induces at least $2^m$ distinct end-to-end labelings on these prompts. Therefore, if one considers all length-$T$ generations that functions in $\cF$ can produce starting from the prompts $x_1,\dots,x_m$, then there must be at least $2^m$ such distinct generations. The key point is therefore to upper bound how many such generations are possible.

A first crude bound, essentially the one used by Joshi et al., is obtained as follows. For each prompt $x_i$, consider the full binary generation tree of depth $T$, consisting of all strings obtained by appending to $x_i$ at most $T$ tokens. Across all $m$ prompts, this gives on the order of $m2^T$ strings. Any possible $T$-step generation produced by a function in $\cF$ is determined by the labels that the function assigns to these strings. Hence, the number of possible length-$T$ generations is at most the number of label patterns that $\cF$ realizes on this collection of roughly $m2^T$ strings. By the Sauer-Shelah-Perles lemma, this number grows polynomially in $m2^T$, with exponent $\VC(\cF)$. Comparing this upper bound with the lower bound $2^m$ coming from shattering yields the general upper bound of \cite{joshi2025theory}, which is linear in $T$.

Our goal is to improve this linear dependence on $T$ to logarithmic dependence under the stronger assumption that the refined tree dimension is finite. The point is that the factor $2^T$ in the argument above is highly wasteful: it counts \emph{all} nodes in the full binary generation tree, even though many of them may be irrelevant for the behavior of the class. What really matters is not the full tree, but only the subtree consisting of branches that are actually realized by functions in $\cF$.

This is exactly where the refined dimension enters. Recall that $\ATdim(\cF)$ rules out the existence of a large embedded leveled perfect subtree inside the realized generation tree. We prove a Sauer-type lemma tailored to this setting: if a depth-$T$ tree contains no such embedded leveled perfect subtree of depth $k$, then the number of leaves - and hence the number of realized branches - is at most polynomial in $T$, with exponent proportional to $k$. Thus, in the counting argument above, one can replace the crude factor $2^T$ by a much smaller quantity that is only polynomial in $T$. 
This is closely related in spirit to the Littlestone-dimension argument of~\cite{joshi2025theory}; our contribution is to show that the same counting principle can be sharpened using the more refined parameter $\ATdim(\cF)$.



Plugging this refined bound into the same comparison as before shows that the size $m$ of any shattered set can grow at most logarithmically with $T$, with dependence proportional to $\ATdim(\cF)$ and $\VC(\cF)$. This yields the desired upper bound on $\VC(\cF^{\etoe,T})$.

Thus, the proof has two parts. The first is a reduction from end-to-end shattering to counting realized generation patterns, following the general framework of \cite{joshi2025theory}. The second is a new combinatorial Sauer-type lemma showing that finite refined tree dimension forces the number of realized branches in a depth-$T$ generation tree to be only polynomial in $T$. It is this second ingredient that drives the exponential improvement in the counting argument, and ultimately yields logarithmic growth.
\medskip

We do not give separate detailed overviews for the remaining results. Proposition~\ref{prop:stable-compression-cot-intro} is proved by a variant of the same inflation-based argument used for Theorem~\ref{thm:cot-main-intro}: a stable compression scheme for $\cF$ can be lifted to one for $\cF^{\ct,T}$ of the same size, after which the conclusion follows from the known sample complexity bound for stable compression schemes. Proposition~\ref{prop:pathological-parity-intro}, in turn, is based on a more ad hoc construction exhibiting a sharp parity effect outside the VC setting. Since these arguments are either close in spirit to the ones already described or more technical than conceptual, we defer them to the formal proofs.

\section{Additional Related Work} \label{sec:related}

The work of \cite{joshi2025theory} who initiated the study of the autoregressive chain-of-thought model, is most related to our work. Their work was inspired by the time-dependent analog of this model, defined and studied in \cite{malach2023auto}. In addition to defining the model, \cite{joshi2025theory} proved some statistical and computational complexity bounds for it. They also discuss its expressivity, and the relation to the Attention mechanism.
The work of \cite{huang2025transformers} studies learnability of chain-of-thought data using transformers. \cite{tsilivis2025reinforcement} studies how learning autoregressive models benefits from reinforcement learning techniques. \cite{shalev2025reasoning} designed a learner for efficient learning from chain-of-thought data. \cite{balcan2025learning} studied the learnability of verifiers for chain-of-thought reasoning mechanisms.
\cite{altabaa2025cot} prove sample complexity bounds for learning with chain-of-thought supervision that are given in terms of a class-dependent \emph{chain-of-thought information measure} that they define.

Motivated by the connection between online learning and statistically-efficient learning of autoregressive models first identified by \cite{joshi2025theory}, the works of \cite{daniely2025online, geneson2025mistake} studied online learnability of deep models.

\section{Preliminaries} \label{sec:prel}

\subsection{Notation}
In the context of sequences, we may use $r^k$ to indicate a string of $k$ many $r$'s.
For any sequence of elements $S = s_1, \ldots s_m$, we denote the first $k$ elements of it by $S_{\leq k}$ or $S[:k]$, and the last $k$ elements by $S[-k:]$. We denote the $i^{th}$ element by $S[i]$ and the $i^{th}$ to last element by $S[-i]$. For a sequence of instances $\boldsymbol{x} = (x_1, \ldots, x_m) \in \cX^{m}$ and a sequence of labels of the same length $\boldsymbol{y} = (y_1,\ldots y_m) \in \cY^m$, we denote the associated sequence of \emph{examples} $(x_1,y_1), \ldots, (x_m,y_m)$ by $(\boldsymbol{x}, \boldsymbol{y})$. A sequence of examples is usually called a \emph{sample}. We may overload set notation such as $\in, \subset$ to sequences as well.

For simplicity of presentation, we may sometimes use asymptotic notation with an $\cF$ subscript, such as $O_\cF(\cdot), \Omega_{\cF}(\cdot), \Theta_{\cF}(\cdot)$ etc.\ to indicate that constants depending on $\cF$ are hidden.

\subsection{Learning autoregressive classes} \label{sec:model-learning}
Following \cite{joshi2025theory}, we study the sample complexity of \emph{learning} an autoregressive class in appropriate analogs of the PAC-learning model \cite{vapnik1971uniform, valiant1984theory}. In the learning model we study, a base class $\cF$ of next-token-generators $f:\cX \to \Sigma$ and a generation length $T \in \mathbb{N}$ are given. Additionally, there is an unknown distribution $D$ over $\cX$, and an unknown target function $f_\star \in \cF$. Our goal is to find a predictor $f: \cX \to \Sigma$ that minimizes the \emph{population loss} with respect to $f_\star^{\etoe,T}$:
\[
L_{D, f_\star^{\etoe,T}}(f) := \Pr_{x \sim D} \mleft[ f(x) \neq  f_\star^{\etoe,T}(x) \mright]. 
\]

There are now two natural PAC-learning analogs for autoregressive classes, both studied in \cite{joshi2025theory}. The first one, called \emph{end-to-end learnability}, amounts to standard PAC learning of the class $\cF^{\etoe,T}$. The standard definition of PAC learning is given below.

\begin{definition}[PAC learnability] \label{def:pac-learnability}
We say that $\cF$ is \emph{PAC-learnable} (or \emph{learnable}) with sample complexity $m:= m(\epsilon,\delta)$, if there exists a learning rule $\Lrn: (\cX \times \cY)^\star \to \cY^{\cX}$, such that for every distribution $D$ over $\cX$, for every $\epsilon, \delta >0$, and for every target function $f_\star \in \cF$, if $\boldsymbol{x}:=x_1, \ldots, x_m \sim_{iid} D$ and $(\boldsymbol{x}, \boldsymbol{y}):= ((x_i, f_\star(x_i))_{i=1}^m$, then
\[
L_{D, f_\star}(f_{\Lrn(\boldsymbol{x}, \boldsymbol{y})}) > \eps
\]
with probability at most $\delta$, where $f_{\Lrn(\boldsymbol{x}, \boldsymbol{y})} := \Lrn(\boldsymbol{x}, \boldsymbol{y})$ and $L_{D, f_\star}(f) := \Pr_{x \sim D} \mleft[ f(x) \neq  f_\star(x) \mright]$.
\end{definition}

We may now formally define end-to-end learnability.

\begin{definition}[End-to-End learnability] \label{def:e2e-learnability}
We say that $\cF$ with generation length $T$ is \emph{end-to-end learnable} (or \emph{$\etoe$-learnable}) with sample complexity $m:= m(\epsilon,\delta)$, if $\cF^{\etoe,T}$ is learnable with sample complexity $m$.
\end{definition}

The other model of interest is \emph{learning with Chain-of-Thought supervision}.

\begin{definition}[Learnability with Chain-of-Thought supervision] \label{def:cot-learnability}
    We say that $\cF$ with generation length $T$ is \emph{learnable with Chain-of-Thought supervision} (or \emph{$\ct$-learnable}) with sample complexity $m:= m(\epsilon,\delta)$, if there exists a learning rule $\Lrn: (\Sigma^\star \times \Sigma^T)^\star \to \Sigma^{\Sigma^\star}$, such that for every distribution $D$ over $\Sigma^\star$, for every $\epsilon, \delta >0$, and for every target function $f_\star \in \cF$, if $\boldsymbol{x}:=x_1, \ldots, x_m \sim_{iid} D$ and $(\boldsymbol{x}, \boldsymbol{y}):= ((x_i, f_\star^{\ct,T}(x_i))_{i=1}^m$, then
    \[
    L_{D, f_\star^{\etoe,T}}(f_{\Lrn(\boldsymbol{x}, \boldsymbol{y})}) > \eps
    \]
    with probability at most $\delta$, where $f_{\Lrn(\boldsymbol{x}, \boldsymbol{y})} := \Lrn(\boldsymbol{x}, \boldsymbol{y})$.
\end{definition}

The difference between the two definitions is that while Definition~\ref{def:e2e-learnability} amounts to standard PAC learning of $\cF^{\etoe,T}$, Definition~\ref{def:cot-learnability} explicitly allows the learner to observe the entire chain-of-thought that produces the final output of $f_\star^{\etoe,T}$. In other words,  Definition~\ref{def:e2e-learnability} measures learnability with $\etoe$ training and $\etoe$ test, and Definition~\ref{def:cot-learnability} measures learnability with $\ct$ training and $\etoe$ test. Thus, there are also two other natural variants that one may consider: learnability with $\etoe$ training and $\ct$ test, and learnability with $\ct$ training and $\ct$ test. While the first variant might not be very useful, the latter is in fact multiclass PAC learning of $\cF^{\ct,T}$, and is at least as difficult as Definition~\ref{def:cot-learnability}. Therefore, any upper bound for it immediately applies to the setting of Definition~\ref{def:cot-learnability}. It is thus useful to define it as well, and apply tools from multiclass PAC learning to handle it (as also observed in \cite{joshi2025theory}).

\begin{definition}[Learnability of full Chain-of-Thought] \label{def:cot-test-learnability}
    We say that $\cF$ with generation length $T$ is \emph{Fully Chain-of-Thought learnable} (or \emph{fully $\ct$-learnable}) with sample complexity $m:= m(\epsilon,\delta)$, if $\cF^{\ct,T}$ is learnable with sample complexity $m$.
\end{definition}

\subsection{Binary trees}

In this paper, a tree $\tree$ is a finite and rooted binary tree, accompanied with the following information:
\begin{enumerate}
    \item Every internal node $v$ is associated with an instance $x \in \cX$, denoted as $x(v)$.
    \item For every internal node $v$, the left outgoing edge is labeled with $0$, and the right outgoing edge with $1$.
\end{enumerate}
The set of vertices and edges of $\tree$ are denoted by $V(\tree)$ and $E(\tree)$, respectively.
Every path in the tree that starts at the root is called a prefix. A prefix that ends at a leaf is called a \emph{branch}. The set of all branches in $\tree$ is denoted $\cB(\tree)$. A prefix can be identified by a sequence of vertices starting from the root, or by just the last vertex in the prefix, or by the associated sequence of edges.
A prefix $p = v_0, \ldots, v_t$ in $\tree$ defines a sample $(x_1,y_1), \ldots, (x_t,y_t) \in \cX \times \cY$ in the natural way: $x_i = x(v_{i-1})$, and $ y_i$ is the label of the edge $(v_{i-1},v_i)$ for all $i$. We denote $\boldsymbol{x}(p) = x_1, \ldots, x_t$ and $\boldsymbol{y}(p) = y_1,\ldots y_t$. For a given vertex $v$, we use $p_\tree(v)$ (or $p(v)$, when the tree $\tree$ is fixed) to denote the path from the root to $v$. We may overload notation and use $p(v)$ to denote the sample associated with the path, when the context is clear. If there is a bijection between the set of vertices and the set of their instance labelings (which is usually the case in this paper), we may replace $v$ with its instance label in all notation. 

We now discuss two types of binary trees used in this paper.

\paragraph{Generation trees.}
Given a  generation length $T$ and an instance $x$, there is a natural way to represent all possible labels that could be given to $x$ by the autoregressive process as a decision tree $\tree(x) := \tree_{ T}(x)$. We call this tree the \emph{generation tree} of $x$ of depth $T$, or simply the \emph{generation tree} of $x$, if $T$ is fixed.  The tree is constructed as follows. The root of $\tree(x)$ is $x$. The left child of $x$ is labeled by $x0$ and likewise, the right child of $x$ is labeled by $x1$. We continue this labeling process inductively on the children of $x$, for a total number of $T$ many times. There is a bijection between the set $\{\boldsymbol{y}(b)\}$ for all branches $b$ in $\tree(x)$ and the label space $\{0,1\}^T$. Furthermore, note that a function $f^{\mathsf{CoT},T} \in \cF^{\mathsf{CoT},T}$ realizes (that is, agrees with) a branch $b$ if and only if $f^{\mathsf{CoT},T}(x) = y(b)$. The branch $b$ may be referred to as the \emph{computation path} of $f^{\mathsf{CoT},T}$ on $x$. This representation will be useful in some of the analyses conducted in this work. See a visualized example of a generation tree in Figure~\ref{fig:arl}.

\begin{figure}
    \centering
\begin{tikzpicture}[
  level distance=25mm,
  sibling distance=55mm,
  level 2/.style={sibling distance=35mm},
  every node/.style={font=\Large, inner sep=1pt},
  edge from parent/.style={draw, line width=0.8pt},
]
\node {0}
  child { node {00}
    child { node {000} }
    child { node {001} }
  }
  child { node {01}
    child { node {010} }
    child { node {011} }
  };
\end{tikzpicture}
    \caption{The tree $\tree_{T}(x)$ for $T=2, x=0$.}
    \label{fig:arl}
\end{figure}
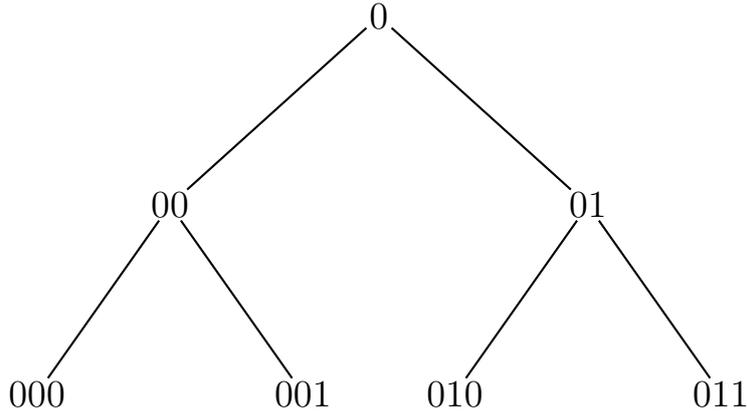




\paragraph{Leveled subtrees.}
If $\tree_1$ and $\tree_2$ are rooted binary trees, we say that $\tree_1$ is a \emph{leveled subtree} of $\tree_2$ if there exists an injective map $\iota:V(\tree_1)\to V(\tree_2)$ such that:
\begin{enumerate}
    \item for every pair of vertices $u,v\in V(\tree_1)$, the vertex $u$ is a left descendant, respectively right descendant, of $v$ in $\tree_1$ if and only if $\iota(u)$ is a left descendant, respectively right descendant, of $\iota(v)$ in $\tree_2$; and
    \item $u$ and $v$ lie on the same level in $\tree_1$ if and only if $\iota(u)$ and $\iota(v)$ lie on the same level in $\tree_2$.
\end{enumerate}

\subsection{Learning theory tools}\label{sec:prel-tools}

Some of our proofs rely on a broad collection of classic learning theory techniques developed through decades of research, including VC-theory, multiclass learning, sample compression schemes, and more.
We define all relevant terms and concepts, and state relevant known results in a self-contained manner.

\subsubsection{Binary class dimensions}
Consider a tree $\tree$ and a class $\cF: \cX \to \{0,1\}$.
Given a function $f\in \cF$, we say that $f$ \emph{realizes} a prefix $p= v_0, \ldots, v_t$ in $\tree$ if $f(\boldsymbol{x}(p)_i) = \boldsymbol{y}(p)_i$ for all $i$.
The set of all functions $f \in \cF$ that realize $p$ are called the \emph{version space} of $v_t$ (or of $p$, or of $(\boldsymbol{x}(p), \boldsymbol{y}(p))$) with respect to $\cF$, which is denoted as $V_{\cF}(v_t)$, or simply $V(v_t)$ when $\cF$ is fixed.
We say that $\cF$ \emph{shatters} $\tree$, if for every leaf $\ell \in \tree$ it holds that $V(\ell) \neq \emptyset$. The set of trees shattered by $\cF$ is denoted by $\cT(\cF)$. We now define two fundamental and useful dimensions.

\begin{definition}[Littlestone dimension]
The \emph{Littlestone dimension} of $\cF$, denoted $\LD(\cF),$ is the maximal depth of a perfect tree which is shattered by $\cF$. If there is no such maximal depth, we say that $\LD(\cF) = \infty$.
\end{definition}

The Littlestone dimension is known to characterize \emph{online learnability}, which is not in the scope of this paper.

Next, we define the VC-dimension, which is known to characterize PAC-learnability. The VC-dimension is defined as the Littlestone dimension, but restricted only to trees where each branch is labeled by the same sequence of instances. A formal definition is given below.

\begin{definition}[VC-dimension]
The \emph{VC dimension} of $\cF$, denoted $\VC(\cF)$, is the maximal depth of a perfect tree which is shattered by $\cF$, and also satisfies $\boldsymbol{x}(b_1) = \boldsymbol{x}(b_2)$ for every two branches $b_1,b_2$. If there is no such maximal depth, we say that $\VC(\cF) = \infty$.
\end{definition}
Since in the context of the VC-dimension we require that $\boldsymbol{x}(b_1) = \boldsymbol{x}(b_2)$ for every two branches $b_1,b_2$, instead of discussing trees, we may simply refer to $\boldsymbol{x}(b_1)$ as a \emph{VC-shattered} (or just \emph{shattered}) set.

We now state the optimal sample complexity bounds for PAC-learning a general function class $\cF: \cX \to \{0,1\}$ in terms of the VC-dimension.
\begin{theorem}[VC-classes sample complexity \cite{blumer1989learnability, ehrenfeucht1989general, hanneke2016optimal}] \label{thm:vc-pac-sample-complexity}
    Let $\cF:\cX \to \{0,1\}$. Then, $\cF$ is learnable if and only if $\VC(\cF) < \infty$. Furthermore, if $\VC(\cF) < \infty$ then $\cF$ is learnable with optimal sample complexity
    \[
    m(\epsilon, \delta) = \Theta \mleft( \frac{\VC(\cF) + \log(1/\delta)}{\epsilon} \mright).
    \]
\end{theorem}
The lower bound was proved by \cite{blumer1989learnability, ehrenfeucht1989general} and the upper bound by \cite{hanneke2016optimal}.

Finite VC-dimension gives an even stronger guarantee than PAC learnability. A typical large enough sample from a distribution $D$ on the domain $\cX$ gives a good approximation on the value of any $f \in \cF$ on an instance drawn from $D$.

\begin{theorem}[VC-classes $\epsilon$-approximation \cite{talagrand1994sharper}] \label{thm:vc-approximation}
    Let $\cF \subset \cY^\cX$ such that $\VC(\cF) < \infty$, and let $D$ be a distribution on $\cX$. Then, for all $\epsilon > 0$ there exists a multiset $X \subset \cX$ of size $|X| = O(\VC(\cF)/\epsilon^2)$ such that for all $f\in \cF$:
    \[
    \mleft \lvert D \mleft ( \{x \in \cX : f(x) = 1\} \mright) - \frac{\lvert \{x \in X : f(x) = 1\} \rvert}{|X|} \mright \rvert \leq \epsilon.
    \]
\end{theorem}

Another useful notion is the \emph{dual VC-dimension}.
\begin{definition}[Dual VC-dimension]
    Given a class $\cF:\cX \to \{0,1\}$, we can define the \emph{dual class} $\cF^\star$ obtained by ``switching" the role between instances and functions. Formally, for every $x \in \cX$ define the function $h_x: \cF \to \{0,1\}$ as
    \[
    h_x(f) = f(x)
    \]
    for all $f \in \cF$.
    Now, define
    \[
    \cF^\star := \{h_x: x \in \cX\}.
    \]
    The \emph{dual VC-dimension} of $\cF$ is  $\VC^\star(\cF):= \VC(\cF^\star)$.
\end{definition}

\subsubsection{Multiclass dimensions}

There are a few useful generalizations of the VC-dimension to the multiclass setting (where the label set $\cY$ has $|\cY| > 2$). We will use the \emph{Natarajan dimension}, denoted by $\Nat(\cdot)$, which is useful in the case where $|\cY| < \infty$ and is defined as follows.

\begin{definition}[Natarajan dimension]
We say that $S := \{x_1, \ldots, x_m\} \subset \cX$ is \emph{Natarajan-shattered} by $\cF: \cX \to \cY$ if there exist two labelings $h_1,h_2 : \cX \to \cY$ so that for all $x \in S$: $h_1(x) \neq h_2(x)$ and in addition, for all $U \subset S$, there exists $f_U \in \cF$ so that
\[
    f_U(x)
    =
    \begin{cases}
        h_1(x) & x \in U, \\
        h_2(x) & x \in S \backslash U.
    \end{cases}
\]
In simple words, any of the possible $2^n$ choices of a label $h_1(x)$ or $h_2(x)$ for every $x \in S$ should be realized by a function from $\cF$. The Natarajan dimension is the maximal size of a Natarajan-shattered set, or $\infty$ if there are such sets of arbitrarily large size.
\end{definition}
Note that when $|\cY| = 2$, a set is Natarajan-shattered if and only if it is VC-shattered, and thus $\Nat(\cF) = \VC(\cF)$.
The Natarajan dimension was shown to characterize learnability of $\cF$ as long as $\cY < \infty$ \cite{natarajan1989learning}.

A useful notion which is closely related with $\Nat(\cF)$ (and with $\VC(\cF)$, for a binary class $\cF$) is the \emph{growth function}.
\begin{definition}[Growth function]\label{def:growth}
    For $\boldsymbol{x} := (x_1, \ldots, x_m) \in \cX^m$ let $f(\boldsymbol{x}) := (f(x_1), \ldots, f(x_m))$ and $\cF(\boldsymbol{x}) := \{f(\boldsymbol{x}): f\in \cF\}$. The \emph{growth function} $\Gamma_{\cF}: \mathbb{N} \to \mathbb{N}$ is defined for every $m$ as:
    \[
    \Gamma_\cF(m) := \max_{\boldsymbol{x}\in\cX^m} |\cF(\boldsymbol{x})|.
    \]
\end{definition}
Note that $\Gamma_\cF(m) \leq |\cY|^m$. For binary classes, equality holds if and only if $m \leq \VC(\cF)$.
We sometimes use the growth function over a set $S \subset \cX$ instead of over a sequence $\boldsymbol{x}$. This is valid since we can just order the set by some arbitrary sequence.

The generalization of the Sauer-Shelah-Perles (SSP) lemma \cite{sauer1972density, shelah1972combinatorial} to the multiclass setting bounds the growth function as follows.
\begin{theorem}[Multiclass SSP \cite{haussler1995generalization}] \label{thm:ssp}
    For all $m\in \mathbb{N}$:
    \[
    \Gamma_\cF(m) \leq (e m |\cY|)^{2\Nat(\cF)}.
    \]
\end{theorem}
Note that in the binary $\cY$ case, this implies $\Gamma_\cF(m) \leq (2 e m)^{2\VC(\cF)}$.

\subsubsection{Sample compression schemes}
Sample compression schemes are a powerful mechanism for deriving sample complexity bounds.
A sample $(\boldsymbol{x},\boldsymbol{y}) \in (\cX \times \cY)^{\star}$ satisfying $y_i = f(x_i)$ for some $f \in \cF$ for all $i\in [k']$ is called \emph{realizable} by the class $\cF$. All the samples we discuss are realizable.
Let $\cS_\cF(k)$ be the set of all realizable samples of size at most $k$. Likewise, let $\cS_\cF(\infty)$ be the set of all finite realizable samples.

\begin{definition}[Sample compression scheme]
A sample compression scheme for $\cF: \cX \to \cY$ with kernel size $k$ and side information represented as a finite set $I$ consists of two mappings: a \emph{compressing map} $\kappa:\cS_\cF(\infty) \to \cS_\cF(k) \times I$ and a \emph{reconstruction map} $\rho: \cS_\cF(k) \times I \to \cY^{\cX}$. The compression maps a sample $(\boldsymbol{x},\boldsymbol{y}) \in (\cX \times \cY)^{\star}$  to  a subsample of it of size at most $k$ and to information $i \in I$, which can be seen as ``instructions" for the reconstruction map as for how to use the subsample. The reconstruction maps a sample of size at most $k$ accompanied with information $i$ to a classifier $f:\cX \to \cY$. The pair $(\kappa, \rho)$ is a \emph{sample compression scheme} if for any $S \in \cS_\cF(\infty)$:
\[
f_{(\rho, \kappa, S)}(x) = y
\]
for all $(x,y) \in S$, where $f_{(\rho, \kappa, S)} := \rho(\kappa(S))$. The size of the scheme is defined as $k + \log |I|$.
\end{definition}
In simple words, compressing and then reconstructing should return a classifier that makes correct predictions on the entire original sample, for all samples. The existence of a sample compression scheme implies a sample complexity bound.

\begin{theorem}[Compression sample complexity bound \cite{littlestone1986relating}] \label{thm:compression-scheme}
    Let $(\kappa, \rho)$ be a sample compression scheme of size $k$ for $\cF$. Then $\cF$ is learnable with sample complexity
    \[
    m(\eps, \delta) = O \mleft( \frac{k \log(1/\epsilon) + \log (1/\delta)}{\eps} \mright).
    \]
\end{theorem}

There is also a special kind of sample compression schemes, called \emph{stable} sample compression schemes, having better guarantees than those of Theorem~\ref{thm:compression-scheme}. For simplicity, we define stable sample compression schemes in the case where $I = \emptyset$, that is, when no side information is involved. This suffices for our usage of this concept.

\begin{definition}[Stable sample compression scheme]
    A compression scheme $(\kappa, \rho)$ is \emph{stable} if for every sample $S \in \cX \times \cY$, and for every $(x,y) \in S \backslash \kappa(S)$ it holds that
    \[
    \kappa(S \backslash (x,y)) = \kappa(S).
    \]
\end{definition}
In simple words, stability amounts to the following condition: if we take a small sample for which the compression map returns the same small sample, then also after adding examples to this small sample, the compression will return the same original small sample. For stable sample compression schemes, we have the following sample complexity bound of \cite{bousquet2020proper}, which is better\footnote{Their result is stated for classes with binary label sets, but holds for any reasonable, and in particular finite label sets.
} than the bound of Theorem~\ref{thm:compression-scheme} by a $\log \frac{1}{\epsilon}$ factor.
\begin{theorem}[Stable compression sample complexity bound \cite{bousquet2020proper}] \label{thm:stable-compression-scheme}
    Let $(\kappa, \rho)$ be a stable sample compression scheme of size $k$ for $\cF$. Then $\cF$ is learnable with sample complexity
    \[
    m(\eps, \delta) = O \mleft( \frac{k + \log (1/\delta)}{\eps} \mright).
    \]
\end{theorem}

\section{End-to-End Learning} \label{sec:e2e-proofs}
In this section, we analyze $\etoe$-learning, that is, the setting described in Definition~\ref{def:e2e-learnability}. Given a base class $\cF: \Sigma^\star \to \Sigma$ and a generation length $T$, this setting amounts to PAC-learning the class $\cF^{\etoe,T}$. Therefore, in light of Theorem~\ref{thm:vc-pac-sample-complexity}, it suffices to analyze $\VC(\cF^{\etoe,T})$, and then Theorem~\ref{thm:vc-pac-sample-complexity} gives matching lower and upper bounds on the sample complexity.

In more detail, we discover the full taxonomy of the VC-dimension of the class $\cF^{\etoe,T}$ for a VC-class $\cF$. We also resolve an open question of \cite{joshi2025theory}, and show that $\etoe$-learnability with $o(T)$ sample complexity can be ensured using a notion which is more relaxed than the Littlestone dimension. On the negative side, we prove that no dimension that characterizes $\etoe$-learnability with $o(T)$ sample complexity exists.

\subsection{The e2e-learning taxonomy of VC-classes} \label{sec:e2e-taxonomy}

We say that a function $r:\mathbb{N}_+ \to \mathbb{N}_+$ is a \emph{monotone-subadditive} growth rate if:
\begin{enumerate}
    \item $r$ is monotone non-decreasing;
    \item $r$ is subadditive.
\end{enumerate}

\begin{theorem} \label{thm:e2e-taxonomy-large-vc}
    Let $r:\mathbb{N_+} \to \mathbb{N}_+$ be a monotone-subadditive rate. Then there exists a base class $\cF_{r}$ such that
    \[
    r(T) \leq \VC(\cF_{r}^{\etoe,T}) \leq r(T) + r(1) \leq  2 r(T)
    \]
    for all $T$.
\end{theorem}

The main ingredient of Theorem~\ref{thm:e2e-taxonomy-large-vc} is the special case $r(1)=1$.

\begin{lemma} \label{lem:e2e-taxonomy}
    Let $r:\mathbb{N}_+ \to \mathbb{N}_+$ be a monotone-subadditive growth rate such that $r(1)=1$. Then there exists a base class $\cF$ such that
    \[
    \VC(\cF^{\etoe,T}) = r(T)
    \qquad\text{for all }T\in\mathbb{N}_+.
    \]
\end{lemma}

The proof proceeds in two steps. First, we show that every set $N\subseteq\mathbb{N}$ gives rise to a class whose $T$-iterated VC dimension is the maximal size of the intersection of $N$ with an interval of length $T$. Second, we show that the functions obtained this way are precisely the monotone-subadditive rates satisfying $r(1)=1$.

For a set $N\subseteq\mathbb{N}$, define
\[
r_N(T):=\max_{u\in\mathbb{N}} |N\cap [u+1,u+T]|.
\]
Thus, $r_N(T)$ is the maximum number of points of $N$ that can lie in an interval of length $T$.

\begin{lemma}\label{lem:taxonomy-step1}
For every $N\subseteq\mathbb{N}$, there exists a base class $\cF(N)$ such that
\[
\VC(\cF(N)^{\etoe,T})=r_N(T)
\qquad\text{for all }T\in\mathbb{N}_+.
\]
\end{lemma}

\begin{proof}
For every $A\subseteq N$ and every $s\in\mathbb{N}$, let $b^{\,s+A}\in\{0,1\}^{\mathbb{N}}$ denote the indicator sequence of the shifted set $s+A$, namely
\[
b_i^{\,s+A}=1 \iff i\in s+A.
\]
Let $f_{s+A}$ be the autoregressive next-bit predictor corresponding to the sequence $b^{\,s+A}$:
\[
f_{s+A}(x):=
\begin{cases}
    b^{\,s+A}_{|x|+1} & x \text{ is a prefix of } b^{\,s+A},\\
    0 & \text{otherwise.}
\end{cases}
\]
Define
\[
\cF(N):=\{f_{s+A}: s\in\mathbb{N},\ A\subseteq N\}.
\]

We prove that for every $T$,
\[
\VC(\cF(N)^{\etoe,T})=r_N(T).
\]

\medskip
\noindent
\textbf{Lower bound.}
Fix $T\in\mathbb{N}_+$, and choose $u\in\mathbb{N}$ such that
\[
|N\cap [u+1,u+T]|=r_N(T).
\]
Define \(I:=N\cap [u+1,u+T]\) and
consider the set of strings
\[
S:=\{0^{t}:t\in I\}.
\]
Thus, \(\lvert S\rvert = r_N(T).\)
We claim that $S$ is shattered by $\cF(N)^{\etoe,T}$.
Indeed, fix an arbitrary labeling $(y_t)_{t\in I}\in\{0,1\}^I$, and define
\[
A:=\{t : t\in I,\ y_t=1\}\subseteq N.
\]
Now consider the function $f_{T+A}\in\cF(N)$ and observe that every string $0^{t}$ with $t \in I$ is a prefix of $b^{\,T+A}$.
Therefore, the end-to-end label after $T$ steps is the bit of $b^{\,T+A}$ at position $T+t$:
\[
f_{T+A}^{\etoe,T}(0^{t})=1
\iff
T+t\in T+A
\iff
t\in A
\iff
y_t=1.
\]
Thus
\[
f_{T+A}^{\etoe,T}(0^{t})=y_t
\qquad\text{for all }t\in I.
\]
So $S$ is shattered, and therefore
\[
\VC(\cF(N)^{\etoe,T})\ge r_N(T).
\]


\medskip
\noindent
\textbf{Upper bound.}
Let $S:=\{x_1,\ldots,x_m\}$ be shattered by $\cF(N)^{\etoe,T}$. We prove that $m\le r_N(T)$.
Since $S$ is shattered, there exists some $f_{s+A}\in\cF(N)$ such that
\[
f_{s+A}^{\etoe,T}(x_i)=1
\qquad\text{for all }i\in[m].
\]
We first claim that every $x_i$ is a prefix of the sequence $b^{\,s+A}$. Indeed, if $x_i$ were not a prefix of~$b^{\,s+A}$, then by definition $f_{s+A}(x_i)=0$, and hence also $f_{s+A}^{\etoe,T}(x_i)=0$, a contradiction.
Thus, all the strings $x_1,\ldots,x_m$ are prefixes of the same infinite sequence $b^{\,s+A}$. Hence they are linearly ordered by the prefix relation. Reordering if necessary, we may assume
\begin{equation} \label{eq:taxonomy-prefix-chain}
x_1 \prec x_2 \prec \cdots \prec x_m \prec b^{\,s+A}.
\end{equation}
Since $x_i$ is a prefix of $b^{\,s+A}$, we have
\[
f_{s+A}^{\etoe,T}(x_i)=b^{\,s+A}_{|x_i|+T}.
\]
Since the left-hand side equals $1$, it follows that
\[
|x_i|+T\in s+A
\qquad\text{for all }i\in[m].
\]
Equivalently, for every $i\in[m]$ there exists $a_i\in A\subseteq N$ such that
\[
|x_i|+T=s+a_i,
\]
that is,
\[
a_i=|x_i|+T-s.
\]
We now claim that
\begin{equation} \label{eq:taxonomy-window}
|x_i|+T \in [|x_m|+1,|x_m|+T]
\qquad\text{for all }i\in[m].
\end{equation}
The upper bound is immediate, since $|x_i|\le |x_m|$.
For the lower bound, it suffices to show that
\[
|x_1|+T\ge |x_m|+1.
\]
Suppose otherwise that $|x_1|+T\le |x_m|$. Since $S$ is shattered, there exists another function $f_{s'+A'}\in\cF(N)$ such that
\[
f_{s'+A'}^{\etoe,T}(x_m)=1
\qquad\text{but}\qquad
f_{s'+A'}^{\etoe,T}(x_1)=0.
\]
Because $f_{s'+A'}^{\etoe,T}(x_m)=1$, the string $x_m$ is a prefix of $b^{\,s'+A'}$. But by~\eqref{eq:taxonomy-prefix-chain}, $x_m$ is also a prefix of $b^{\,s+A}$. Therefore the first $|x_m|$ bits of the two sequences are identical. Since $|x_1|+T\le |x_m|$, the bit at position $|x_1|+T$ is the same in both sequences. Hence
\[
f_{s'+A'}^{\etoe,T}(x_1)=f_{s+A}^{\etoe,T}(x_1),
\]
contradicting the fact that one is $0$ and the other is $1$. This proves~\eqref{eq:taxonomy-window}.

By~\eqref{eq:taxonomy-window}, for every $i\in[m]$ we have
\[
a_i=|x_i|+T-s \in [|x_m|-s+1,|x_m|-s+T].
\]
Thus all the points $a_i$ lie in an interval of length $T$.
Moreover, the $a_i$'s are distinct, because by~\eqref{eq:taxonomy-prefix-chain} the strings $x_i$ form a strict prefix chain and hence have distinct lengths.
Since all the $a_i$'s belong to $N$, we conclude that $N$ contains at least $m$ distinct elements in some interval of length $T$. By the definition of $r_N(T)$, this implies
\[
m\le r_N(T).
\]
Therefore
\[
\VC(\cF(N)^{\etoe,T})\le r_N(T).
\]
Combining the lower and upper bounds yields
\[
\VC(\cF(N)^{\etoe,T})=r_N(T)
\qquad\text{for all }T\in\mathbb{N}_+.
\]
\end{proof}

We now characterize the possible functions $r_N$.

\begin{lemma}\label{lem:taxonomy-step2}
A function $r:\mathbb{N}_+\to\mathbb{N}_+$ is of the form $r=r_N$ for some nonempty set $N\subseteq\mathbb{N}$ if and only if $r$ is monotone non-decreasing, subadditive, and satisfies $r(1)=1$. The constant zero function is realized by $N=\emptyset$.
\end{lemma}

\begin{proof}
We first prove necessity. Let $N\subseteq\mathbb{N}$, and consider $r_N$.
The function $r_N$ is monotone non-decreasing, since enlarging the interval length can only increase the maximum possible intersection size.
It is also subadditive. Indeed, every interval of length $T+S$ can be decomposed into two consecutive intervals of lengths $T$ and $S$. Therefore, for every $u\in\mathbb{N}$,
\[
|N\cap [u+1,u+T+S]|
\le
|N\cap [u+1,u+T]| + |N\cap [u+T+1,u+T+S]|.
\]
Taking the maximum over $u$ gives
\[
r_N(T+S)\le r_N(T)+r_N(S).
\]
Finally, if $N\neq\emptyset$, then $r_N(1)=1$: every interval of length $1$ contains at most one natural number, and some such interval contains an element of $N$. If $N=\emptyset$, then clearly $r_N\equiv 0$.

We now prove sufficiency. Let $r:\mathbb{N}_+\to\mathbb{N}_+$ be monotone non-decreasing and subadditive, and assume that $r(1)=1$. We will construct a set $N\subseteq\mathbb{N}$ such that
\[
r_N(T)=r(T)
\qquad\text{for all }T\in\mathbb{N}_+.
\]
Define
\[
K:=\sup_T r(T)\in \mathbb{N}\cup\{\infty\}.
\]
For each $k\in [K]$, where by convention $[K]:=\mathbb{N}$ when $K=\infty$, define
\[
t_k:=\min\{T\in\mathbb{N}_+ : r(T)\ge k\}.
\]
Let
\[
N:=\{t_k : k\in [K]\}.
\]
We claim that $r_N(T)=r(T)$ for all $T$.
First note that subadditivity and the identity $r(1)=1$ imply
\[
r(T+1)\le r(T)+r(1)=r(T)+1.
\]
Since $r$ is also monotone non-decreasing, it follows that
\[
r(T+1)-r(T)\in\{0,1\}
\qquad\text{for all }T\in\mathbb{N}_+.
\]
In particular, the values $t_k$ are all distinct.
Moreover, for every $T\in\mathbb{N}_+$ we have
\[
|\{k : t_k\le T\}|=r(T).
\]
Equivalently,
\[
|N\cap [1,T]|=r(T).
\]
This immediately yields
\[
r_N(T)\ge r(T),
\]
since the interval $[1,T]$ has length $T$ and contains exactly $r(T)$ elements of $N$.

For the reverse inequality, fix $u\in\mathbb{N}$. Since the $t_k$'s are distinct and increasing, we have
\[
|N\cap [u+1,u+T]|
=
|\{k : t_k\in [u+1,u+T]\}|
=
|\{k : t_k\le u+T\}| - |\{k : t_k\le u\}|.
\]
By the identity above, this equals
\[
r(u+T)-r(u).
\]
By subadditivity,
\[
r(u+T)-r(u)\le r(T).
\]
Thus every interval of length $T$ contains at most $r(T)$ elements of $N$, and therefore
\[
r_N(T)\le r(T).
\]
Combining the two inequalities, we conclude that
\[
r_N(T)=r(T)
\qquad\text{for all }T\in\mathbb{N}_+.
\]
This proves the sufficiency. The zero function is realized by taking $N=\emptyset$.
\end{proof}

We may now prove Lemma~\ref{lem:e2e-taxonomy}.

\begin{proof}[Proof of Lemma~\ref{lem:e2e-taxonomy}]
By Lemma~\ref{lem:taxonomy-step2}, since $r$ is monotone non-decreasing, subadditive, and satisfies $r(1)=1$, there exists a set $N\subseteq\mathbb{N}$ such that
\[
r(T)=r_N(T)
\qquad\text{for all }T\in\mathbb{N}_+.
\]
Then Lemma~\ref{lem:taxonomy-step1} yields a base class $\cF(N)$ satisfying
\[
\VC(\cF(N)^{\etoe,T})=r_N(T)=r(T)
\qquad\text{for all }T\in\mathbb{N}_+.
\]
Taking $\cF:=\cF(N)$ completes the proof.
\end{proof}

We now prove the extension of Lemma~\ref{lem:e2e-taxonomy} for the case $r(1) > 1$. The main tool we use is the \emph{cartesian product class} of domain-disjoint classes. The cartesian product class of domain-disjoint classes is defined over two classes, $\cH_1 \subset \{0,1\}^{\cX_1}$ and $\cH_2 \subset \{0,1\}^{\cX_2}$, where $\cX_1 \cap \cX_2 = \emptyset$. Let $\cX$ be any domain such that $\cX_1 \cup \cX_2 \subset \cX$. Unless stated otherwise, we use $\cX = \{0,1\}^\star$. Then the cartesian product class of $\cH_1, \cH_2$ over the domain $\cX$ is denoted by $\cH_1 \uplus \cH_2$, and defined as
\[
\cH_1 \uplus \cH_2 := \{h_v : v \in \cH_1 \times \cH_2\},
\]
where each $h_v$ is defined as:

\[
h_v(x)
:=
\begin{cases}
    v_i(x) & x \in \cX_i, i \in \{1,2\}, \\
    0      & \text{Otherwise.}
\end{cases}
\]

We use the following lemma, which is very similar to a result stated in \cite[Lemma 16]{doliwa2014recursive}. For completeness, we prove it here.

\begin{lemma} \label{lem:cartesian}
    Let $\cH_1 \subset \{0,1\}^{\cX_1}, \cH_2 \subset \{0,1\}^{\cX_2}$ where $\cX_1 \cap \cX_2 = \emptyset$. Then
    \[
    \VC(\cH_1 \uplus \cH_2) = \VC(\cH_1) + \VC(\cH_2).
    \]
\end{lemma}

\begin{proof}
    For the lower bound, let $S_1 \subset \cX_1$ and $S_2 \subset \cX_2$ be sets shattered by $\cH_1$ and $\cH_2$, of sizes  $\VC(\cH_1)$ and $\VC(\cH_2)$, respectively. Let $S = S_1 \cup S_2$. Since $\cX_1 \cap \cX_2 = \emptyset$, we have $|S| = \VC(\cH_1) + \VC(\cH_2)$. It remains to show that $S$ is shattered by $\cH_1 \uplus \cH_2$. Let $\boldsymbol{y^{(1)}}, \boldsymbol{y^{(2)}}$ be labelings of $S_1, S_2$. Since $S_1, S_2$ are shattered, there are $h_1 \in \cH_1, h_2 \in \cH_2$ realizing $\boldsymbol{y^{(1)}}, \boldsymbol{y^{(2)}}$. Thus, $h_{h_1,h_2} \in \cH_1 \uplus \cH_2$ realizes the labeling of $S$ given by the labelings $\boldsymbol{y^{(1)}}, \boldsymbol{y^{(2)}}$.

    For the upper bound, let $S$ be a set shattered by $\cH_1 \uplus \cH_2$. First note that for any $x \notin \cX_1 \cup \cX_2$, the class $\cH_1 \uplus \cH_2$ is constant $0$, so $S \subset \cX_1 \cup \cX_2$. As $\cX_1 \cap \cX_2 = \emptyset$, we may partition $S$ to $S = S_1 \cup S_2$ where $S_1 \subset \cX_1$, $S_2 \subset \cX_2$, and $S_1 \cap S_2 = \emptyset$. By definition, for every labeling of $S$ there are $h_1 \in \cH_1, h_2 \in \cH_2$ such that $h_{h_1,h_2} \in \cH_1 \uplus \cH_2$ realized this labeling. In particular, for every labeling of $S_1$ there exists $h_1 \in \cH_1$ that realized it. Therefore, $S_1$ is shattered by $\cH_1$ and thus $S_1 \leq \VC(\cH_1)$. The same argument holds for $S_2$, and thus $|S| \leq \VC(\cH_1) + \VC(\cH_2)$.
\end{proof}

Lemma~\ref{lem:cartesian} can of course be extended to handle a product of any finite number of classes.
We may now prove Theorem~\ref{thm:e2e-taxonomy-large-vc}.

\begin{proof}[Proof of Theorem~\ref{thm:e2e-taxonomy-large-vc}]
    Let $r$ be a monotone-subadditive rate. Define the normalized rate as
    \[
    \tilde{r}(T) = \ceil*{\frac{r(T)}{r(1)}}
    \]
    for all $T$. First, we claim that $\tilde{r}$ satisfies the conditions of Lemma~\ref{lem:e2e-taxonomy}. It is clear that $\tilde{r}(1)=1$, so it remains to show that $\tilde{r}$ is non-decreasing and subadditive. Let $n,m \in \mathbb N$. So
    \[
    \tilde{r}(n + m)
    =
    \ceil*{\frac{r(n+m)}{r(1)}}
    \geq
    \ceil*{\frac{r(n)}{r(1)}}
    =
    \tilde{r}(n)
    \]
    where the inequality is since $r$ is non-decreasing. Thus, $\tilde{r}$ is non-decreasing. Furthermore:
    \[
    \tilde{r}(n + m)
    =
    \ceil*{\frac{r(n+m)}{r(1)}}
    \leq
    \ceil*{\frac{r(n) + r(m)}{r(1)}}
    \leq
    \ceil*{\frac{r(n)}{r(1)}} + \ceil*{\frac{r(m)}{r(1)}}
    =
    \tilde{r}(n) + \tilde{r}(m),
    \]
    where the first inequality is by subadditivity of $r$, and the second inequality is by the bound $\ceil*{x + y} \leq \ceil*{x} + \ceil*{y}$ for all $x,y \in \mathbb{R}_+$.
    Thus, Lemma~\ref{lem:e2e-taxonomy} applies for $\tilde{r}$, and there exists a class $G: \{0,1\}^\star \to \{0,1\}$, such that $\VC(G^{\etoe,T}) = \tilde{r}(T)$ for all $T$.

    Denote $r(1) = c$. We now create $c$ many copies of $G$ defined over disjoint domains. Towards this end, define the following set of prefix-incomparable strings:
    \[
    P : =\{p_i : i \in [c]\},
    \]
    where $p_i = 0^i 1$. Note that indeed for all $p_i,p_j \in P$, $p_i$ is not a prefix of $p_j$. For each $p_i \in P$, define the domain
    \[
    \cX_i := \{p_i x: x\in \{0,1\}^\star\},
    \]
    and the class $G_i \subset \{0,1\}^{\cX_i}$ by
    \[
    G_i := \{g^{(i)} : g \in G\}
    \]
    where $g^{(i)}$ is defined as:
    \[
    g^{(i)}(p_i x) = g(x)
    \]
    for all $x\in \{0,1\}^\star$.
    Now, let $\cF: \{0,1\}^\star \to \{0,1\}$ be the cartesian product class of $G_1, \ldots, G_c$.
    Since the autoregressive process only appends bits to the end of the input string and does not affect existing bits of the input, it is clear that $\VC(G_i^{\etoe,T}) = \VC(G^{\etoe,T})$ for all $i,T$. Now, we claim that
    \begin{equation} \label{eq:cartes-eq}
        \cF^{\etoe,T} = G_1^{\etoe,T} \uplus \ldots \uplus G_c^{\etoe,T}.
    \end{equation} 
    Having \eqref{eq:cartes-eq} in hand, we are done, since Lemma~\ref{lem:cartesian} implies that
    \[
    \VC(\cF^{\etoe,T}) = \sum_{i \in [c]} \VC(G^{\etoe,T}_i)
    =
    c \cdot \tilde{r}(T)
    =
    c \cdot \ceil*{r(T)/c},
    \]
    and thus
    \[
    r(T) \leq \VC(\cF^{\etoe,T}) \leq c (r(T)/c + 1) = r(T) + r(1) \leq 2 r(T),
    \]
    as required. The rightmost inequality follows because $r$ is nondecreasing.
    
    So, it remains to prove \eqref{eq:cartes-eq}.
    
    Let $f_v^{\etoe,T} \in \cF^{\etoe,T}$, where $v \in G_1 \times \ldots \times G_c$. Thus, for any instance $x$ we have $f_v^{\etoe,T}(x) = v_i^{\etoe}(x)$ if $x \in \cX_i$ for some $i$, and otherwise $f_v^{\etoe,T}(x) = 0$. Therefore, $f_v^{\etoe,T} \in G_1^{\etoe,T} \uplus \ldots \uplus G_c^{\etoe,T}$ by definition, and so $\cF^{\etoe,T} \subset G_1^{\etoe,T} \uplus \ldots \uplus G_c^{\etoe,T}$.
    
    Now, let $g_{v} \in G_1^{\etoe,T} \uplus \ldots \uplus G_c^{\etoe,T}$. By definition, $v \in G_1^{\etoe,T} \times \ldots \times G_c^{\etoe,T}$, and for any $x$, we have $g_v(x) = v_i(x)$ if $x \in \cX_i$ and $g_v(x) = 0$ otherwise. By definition, for every $i$ there exists $h_i \in G_i$ such that $v_i = h_i^{\etoe,T}$. For $h:= (h_1, \ldots, h_c) \in G_1 \times \ldots \times G_c$, let $f_h$ be defined for any $x$ by $f_h(x) = h_i(x)$ if $x \in \cX_i$ for some $i$, and $f_h(x) = 0$ otherwise. By definition, $f_h \in \cF$ and thus $f_h^{\etoe,T} \in \cF^{\etoe,T}$. Therefore, if $f_h^{\etoe,T} = g_v$ then we are done. Let us show that indeed $f_h^{\etoe,T} = g_v$. Let $x \in \cX_i$ for some $i$, then
    \[
    f_h^{\etoe,T}(x) = h_i^{\etoe,T}(x) = v_i(x) = g_v(x).
    \]
    Otherwise, we have
    \[
    f_h^{\etoe,T}(x) = 0 = g_v(x).
    \]
    Therefore, we also have $G_1^{\etoe,T} \uplus \ldots \uplus G_c^{\etoe,T} \subset \cF^{\etoe,T}$.
\end{proof}


\subsection{No characterization of sublinear rates exists}
Recall the definition of a characterization of sublinear rates.
\begin{definition}[Characterizing sublinear rates] \label{def:sublinear-rates}
Let $\Dim$ be any function that assigns to each base class $\cF$ either a natural number or $\infty$. We say that $\Dim$ \emph{characterizes sublinear rates} if there exists a function \(M:\mathbb{N}\times\mathbb{N}\to\mathbb{R}_{\ge 0}\) such that for every base class $\cF$, the following hold:
\begin{enumerate}
    \item $\Dim(\cF)<\infty$ if and only if the function $T\mapsto \VC(\cF^{\etoe,T})$ is sublinear in $T$;
    \item if $\Dim(\cF)<\infty$, then for every $T\ge 1$,
    \[
    \VC(\cF^{\etoe,T}) \le M(\Dim(\cF),T);
    \]
    \item for every fixed $d\in\mathbb{N}$, the function $T\mapsto M(d,T)$ is sublinear in $T$.
\end{enumerate}
\end{definition}

Below, we formally prove that no such dimension exists.
Knowing Lemma~\ref{lem:e2e-taxonomy}, this is not very surprising: Lemma~\ref{lem:e2e-taxonomy} establishes a very rich and dense landscape of possible rates. In light of Definition~\ref{def:sublinear-rates}, for any bound $M(\Dim(\cF), T)$ we may hope to hold via some dimension, we can choose a rate $r$ that contradicts this bound, while still being sublinear. Below, we formally establish this idea via a diagonalization argument.

\begin{theorem}\label{thm:no-dimension-sublinear-proof}
There is no dimension that characterizes sublinear rates according to Definition~\ref{def:sublinear-rates}.
\end{theorem}

\begin{proof}
The proof is via a standard diagonalization argument. We assume towards contradiction that there exists a dimension $\Dim$ that characterizes sublinear rates with the function $M(d,T)$. 
To prove the theorem, we will construct a rate $r$ such that
\begin{enumerate}
    \item $r$ is sublinear.
    \item Lemma~\ref{lem:e2e-taxonomy} can be applied to $r$.
    \item For every $d$, $r(T_d) \gg M(d,T_d)$ for some $T_d \in \mathbb{N}$.
\end{enumerate}
Having such a rate $r$, we are done: Let $\cF$ be a base class guaranteed by Lemma~\ref{lem:e2e-taxonomy} to realize the rate $r$. Since $r$ is sublinear, we have $\Dim(\cF) = d$ for some $d \in \mathbb{N}$. Thus, it should hold that $\VC(\cF^{\etoe,{T_d}}) \leq M(d,{T_d})$. However, by Lemma~\ref{lem:e2e-taxonomy} it should also hold that $\VC(\cF^{\etoe,{T_d}}) \approx r(T_d) \gg M(d,T_d)$, which is a contradiction. It remains to construct $r$, and show the argument sketched above formally, with the constructed rate $r$.

We first construct $r$. We start by constructing a sequence $\{T_d\}_{d \geq 0}$ of values for $T$ for which the bound for $M(d,T_d)$ will be contradicted. Since the mapping $T \mapsto M(d,T)$ is sublinear, for every $d \geq 1$ there exists $N_d \geq 1$ such that for all $T \geq N_d$ we have $\frac{M(d,T)}{T} < 1/2^{d+ 3}$. Define $T_0 = 1$, and for every $d > 0$, define
\[
T_d := \max\{4 T_{d-1}, N_d\}.
\]
We can now define the rate $r$. We first define a rate $\tilde{r}: \mathbb{N}_+ \to \mathbb{R}_+$, and then we define $r$ by $r(T) = \ceil*{\tilde{r}(T)}$ for all $T$. First, let $\tilde{r}(T_0) = \tilde{r}(1) = 1$. For all $d \geq 1$, define $\tilde{r}(T_d) = T_d/2^d$. Now, for any $d \geq 0$, define $\tilde{r}$ in the range  $[T_d, T_{d+1}]$ by a linear interpolation between the points $(T_d, \tilde{r}(T_d)), (T_{d+1}, \tilde{r}(T_{d+1}))$. We first show that $\tilde{r}$ is non-decreasing and subadditive, starting with $\tilde{r}$ being non-decreasing. For all $d > 0$, we have $T_{d} \geq 4 T_{d-1}$, so
    \[
    \tilde{r}(T_d) = T_d/2^d
    \geq
    \frac{4 T_{d-1}}{2^d}
    =
    \frac{4 T_{d-1}}{2 \cdot 2^{d-1}}
    =
    2 \tilde{r}(T_{d-1})
    >
    \tilde{r}(T_{d-1}).
    \]
    Furthermore, in the range $[T_{d-1}, T_d]$, $\tilde{r}$ is defined by a linear interpolation so $\tilde{r}$ is non-decreasing.
    We now show that $\tilde{r}$ is subadditive. For all $d$, we have
    \[
    \frac{\tilde{r}(T_d)}{T_d} = \frac{T_d/2^d}{T_d} = 1/2^d.
    \]
    Using again the fact that in the range $[T_{d-1}, T_d]$, $\tilde{r}$ is defined by a linear interpolation, we get that the function $\tilde{r}(T)/T$ is non-increasing. This implies that $\tilde{r}(T)$ is subadditive. Indeed, let $T_1, T_2 \in \mathbb{N}$ so that $T_1 \geq T_2$. Then:
    \[
    \tilde{r}(T_1 + T_2)
    \leq
    (T_1 +T_2) \tilde{r}(T_1)/T_1
    =
    \tilde{r}(T_1) + T_2  \tilde{r}(T_1) / T_1
    \leq
    \tilde{r}(T_1) + T_2  \tilde{r}(T_2) / T_2
    =
    \tilde{r}(T_1) + \tilde{r}(T_2),
    \]
    We now claim that $r := \ceil*{\tilde{r}}$ is also non-decreasing and subadditive. It is clear that $r$ is non-decreasing. To see that it is subadditive, observe that
    \[
    r(T_1 +T_2)
    =
    \ceil*{\tilde{r}(T_1+T_2)}
    \leq
    \ceil*{\tilde{r}(T_1) + \tilde{r}(T_2)}
    \leq
    \ceil*{\tilde{r}(T_1)} + \ceil*{\tilde{r}(T_2)}.
    \]
    We may now deduce the claimed result. Let $\cF$ be the base class guaranteed by Lemma~\ref{lem:e2e-taxonomy} to realize $r$. Then $\Dim(\cF) = d$ for some $d \in \mathbb{N}$. Therefore:
    \begin{equation} \label{eq:no-dimension-upper}
        \VC(\cF^{\etoe,{T_d}}) \leq M(d,{T_d}) < \frac{T_d}{2^{d+3}}.
    \end{equation}
    On the other hand, by Lemma~\ref{lem:e2e-taxonomy} we have
    \begin{equation}\label{eq:no-dimension-lower}
        \VC(\cF^{\etoe,{T_d}})
        =
         r(T_d)
        =
        \ceil*{\frac{T_d}{2^d}}
        \geq
        \frac{T_d}{2^d}.
    \end{equation}
    Combining \eqref{eq:no-dimension-upper} and \eqref{eq:no-dimension-lower}, we obtain
    \[
    \frac{T_d}{2^{d}} < \frac{T_d}{2^{d+3}},
    \]
    which is clearly a contradiction. This concludes the proof.
\end{proof}

\subsection{A sufficient condition for sublinear sample complexity} \label{sec:e2e-sufficient} \label{sec:art}

\cite{joshi2025theory} proved that while for some base VC-classes $\cF$ it holds that $\VC(\cF^{\etoe,T}) = \Omega(T)$, all base classes $\cF$ satisfy $\VC(\cF^{\etoe,T}) = O( \LD(\cF) \log T)$. That is, unlike finite VC-dimension, finite Littlestone dimension ensures that the growth of $\VC(\cF^{\etoe,T})$ with respect to $T$ is at most logarithmic. Naturally, they asked if $\VC(\cF^{\etoe,T}) = o(T)$ can be ensured using a notion which is more relaxed than the Littlestone dimension. We prove that the \emph{autoregressive tree dimension} of $\cF$, denoted as $\ARL(\cF)$, provides a positive answer to their question. Let us define $\ARL(\cF)$.

\begin{definition}[Autoregressive tree dimension] \label{def:art}
Let $\cF \subseteq \{0,1\}^{\{0,1\}^\star}$. The \emph{autoregressive tree dimension} of $\cF$, denoted $\ATdim(\cF)$, is the largest integer $d$ for which there exists $x \in \{0,1\}^\star$ and a generation length $T\ge 1$ such that $\tree_T(x)$ contains a perfect leveled subtree of depth $d$ which is shattered by $\cF$. If such subtrees exist for arbitrarily large $d$, we define $\ATdim(\cF)=\infty$.
\end{definition}

The following theorem establishes that finiteness of $\ARL (\cF)$ for any VC-class $\cF$ guarantees that $\VC(\cF^{\etoe,T})$ grows at most logarithmically in $T$:

\begin{theorem} \label{thm:arl-upper-bound}
    Let $\cF$ be a class with $\max\{ \VC(\cF), \ARL(\cF) \} < \infty$.
    Then for all $T \geq 2$:
    \[
    \VC(\cF^{\etoe,T}) = O_{\cF}(\log T).
    \]
\end{theorem}

To give a positive answer to the question of \cite{joshi2025theory}, we need to establish that requiring $\max\{ \VC(\cF), \ARL(\cF) \} < \infty$ is indeed a strictly weaker requirement than $\LD(\cF) < \infty$. This is proved in the following proposition. 

\begin{proposition} \label{prop:e2e-relaxed}
    For any class $\cF$, we have
    \[
    \max\{\VC(\cF), \ARL(\cF)\} \leq \LD(\cF).
    \]
    Furthermore, there exists a class $\cF$ such that:
    \begin{enumerate}
        \item $\VC(\cF) = \ARL(\cF) = 1$.
        \item $\LD(\cF) = \infty$.
    \end{enumerate}
\end{proposition}

We note that $\max\{ \VC(\cF), \ARL(\cF) \} < \infty$ is only a sufficient condition for sublinear  growth of $\VC(\cF^{\etoe,T})$. Indeed, as shown in Lemma~\ref{lem:e2e-taxonomy}, many growth rates that are both super-logarithmic and sublinear in $T$ are attainable, and yet Theorem~\ref{thm:arl-upper-bound} implies that those classes have infinite autoregressive tree dimension.


In the following three sections, we prove Theorem~\ref{thm:arl-upper-bound} and Proposition~\ref{prop:e2e-relaxed}.

\subsubsection{Proof of Theorem~\ref{thm:arl-upper-bound}}.

For simplicity of calculations, we assume that $\ARL(\cF) \geq 1$.
To prove Theorem~\ref{thm:arl-upper-bound}, we first need to prove a Sauer-Shelah-Perles lemma (see Theorem~\ref{thm:ssp}) style bound for leveled subtrees. Note that a leveled subtree is a more restrictive notion of the standard model-theoretic definition of an \emph{embedded} tree. An embedded tree has the same definition, except that for a classic embedded tree, it is not required that the injective map defining the embedded tree preserves level relations. We will refer to standard embedded trees simply as \emph{subtrees}, in contrast with leveled subtrees, that are required to  preserve level relations. For standard embedded trees, there is a known Sauer-Shelah-Perles lemma (see e.g.\ \cite{walker2022tree}). We prove a slightly stronger version that gives the same bound for leveled subtrees. For a tree $\tree$ and a class $\cF$, let $\cB_\cF(\tree)$ be the set of branches in $\tree$ which are realized by $\cF$. Note that a set of branches is itself a subtree of the original tree.

\begin{lemma} \label{lem:ul-embedded-ssp}
    Let $d\in \mathbb{N}$, and let $\tree$ be a tree of depth $T$, such that the maximal depth of a perfect leveled subtree in $\tree$ is $d$. Then the number of leaves in $\tree$ is at most $\binom{T}{\leq d}$.
\end{lemma}

\begin{proof}
    If $d \geq T$, then $\binom{T}{\leq d} = 2^T \leq 2^d$, and then the bound is trivial. So, suppose that $d < T$.
    We prove the claim for all $T > d \geq 0$ by induction on $T$. For $T = 0$, it is impossible that $d < T$, so the base case holds.

    We now prove the induction step.
    Suppose without loss of generality that all leaves of $\tree$ are at depth $T$. Let $\tree_1$ be the tree obtained from truncating level $T$ of $\tree$. There are two types of leaves of $\tree_1$: leaves that in $\tree$ have two children, and leaves that in $\tree$ have one child. Let $\tree_2$ be the subtree of $\tree_1$ obtained by removing all branches of $\tree_1$ whose leaves have one child in $\tree$. First, note that
    \begin{equation} \label{eq:sum_of_branches}
        |\cB(\tree)| = |\cB(\tree_1)| + |\cB(\tree_2)|.
    \end{equation}
    We now bound each term in the RHS of \eqref{eq:sum_of_branches}.
    By the induction hypothesis, we have
    \begin{equation} \label{eq:first_subtree_bound}
        |\cB(\tree_1)| \leq \binom{T-1}{\leq d}.
    \end{equation}
    
    We now claim that $\tree_2$ does not contain a perfect leveled subtree of depth $d$. Let us assume that it does. Now, take a leaf of this perfect leveled subtree. Starting from the node in $\tree_2$ corresponding to this leaf, we can move along a branch of $\tree_2$ to depth $T-1$, reaching a leaf of $\tree_2$. By the definition of $\tree_2$, this leaf has two children in $\tree$, which can be added to the leveled subtree, and this tree is still a subtree of $\tree$. We can do the same for all leaves of the leveled subtree, and construct a leveled subtree of $\tree$, since all nodes we add lie at the same depth $T$ in $\tree$. We got a perfect leveled subtree of $\tree$ of depth $d+1$, which is a contradiction. Therefore,  $\tree_2$ does not contain a perfect leveled subtree of depth $d$. By induction hypothesis:
     \begin{equation} \label{eq:second_subtree_bound}
        |\cB(\tree_2)| \leq \binom{T-1}{\leq d-1}.
    \end{equation}
    Using \eqref{eq:sum_of_branches}, \eqref{eq:first_subtree_bound}, \eqref{eq:second_subtree_bound} and Pascal's inequality, we obtain
    \[
    |\cB(\tree)| \leq \binom{T-1}{\leq d} + \binom{T-1}{\leq d-1} = \binom{T}{ \leq d},
    \]
    as required.
\end{proof}

\begin{lemma} \label{lem:autoregressive-ssp}
    Let $\cF$ be a class. Then for any $x\in \cX$ and any $T \in \mathbb{N}$:
    \[
    |\cB_\cF(\tree_T(x))| \leq T^{2\ARL(\cF)}.
    \]
\end{lemma}

\begin{proof}
    Let $x\in \cX$ and $T \in \mathbb{N}$. By definition of $\ARL(\cF)$, the maximal depth of a perfect leveled subtree of $\cB_\cF(\tree_T(x))$ is $\ARL(\cF)$. By Lemma~\ref{lem:ul-embedded-ssp},
    \[
    |\cB_\cF(\tree_T(x))| \leq {\binom{T}{ \leq \ARL(\cF)}} \leq T^{2\ARL(\cF)},
    \]
    as claimed.
\end{proof}


We may now prove Theorem~\ref{thm:arl-upper-bound}.

\begin{lemma} \label{lem:arl-upper-bound}
    Let $\cF$ be a class. If $T \geq 20 (\ARL(\cF) \VC(\cF))$  then:
    \[
    \VC(\cF^{\etoe,T}) \leq  20 (\ARL(\cF) \VC(\cF)) \log T.
    \]
    In the complementing case, $\VC(\cF^{\etoe,T}) = \Tilde{O}(\ARL(\cF) \VC(\cF))$.
\end{lemma}

\begin{proof}
    The proof follows the same lines of the proof of \cite[Theorem B.1]{joshi2025theory}, with a crucial modification: we may use the bound of Lemma~\ref{lem:autoregressive-ssp} instead of the trivial bound $|\cB_\cF(\tree_T(x))| \leq 2^T$ used in \cite{joshi2025theory}, which is avoidable when $\ARL(\cF) < \infty$. We will now formally prove the lemma.

    Let $S$ be a set that is shattered by $\cF^{\etoe,T}$, and suppose that $|S| = m \geq 2$. Therefore,
    \begin{equation} \label{eq:shattered-lower-bound}
        \Gamma_{\cF^{\etoe,T}}(m) = 2^m.
    \end{equation}
    We now upper bound $\Gamma_{\cF^{\etoe,T}}(m)$ in terms of $\Gamma_{\cF}$. First, for every instance $x\in \cX$, let
    \[
    V_{\cF}(x,T) = \{ x'\in v(\tree_T(x)): p(x') \in \cS_\cF(\infty) \}.
    \]
    In words, $V_{\cF}(x,T)$ is the set of all instances $x'$ labeling the nodes of $\tree_T(x)$, such that the path from $x$ to $x'$ is realized by $\cF$. Now, define
    \[
    S_{\cF}(T) = \bigcup_{x \in S} V_{\cF}(x,T).
    \]
    In words, $S_{\cF}(T)$ is the extension of $S$ such that for every $x \in S$, we add all instances of $V_{\cF}(x,T)$ to $S_{\cF}(T)$. Now note that there exists an injective mapping $M$ from all labelings of $S$ by $\cF^{\etoe}$, denoted  $\cF^{\etoe,T}(S)$,  to all labelings of $S_{\cF}(T)$ by $\cF$, denoted $\cF(S_{\cF}(T))$. Indeed, let $\boldsymbol{\ell_T} \in \cF^{\etoe,T}(S)$, and let $f^{\etoe,T} \in \cF^{\etoe,T}$ that realizes $\boldsymbol{\ell_T}$. Then, the mapping $M$ sends $\boldsymbol{\ell_T}$ to $\boldsymbol{\ell} \in \cF(S_{\cF}(T))$, such that $\boldsymbol{\ell}$ is realized by the function $f\in \cF$ that defines $f^{\etoe,T}$ by the autoregressive process. Now, note that if $\boldsymbol{\ell_T},\boldsymbol{\ell_T}' \in \cF^{\etoe,T}(S)$ are both mapped to the same $\boldsymbol{\ell} \in \cF(S_{\cF}(T))$, then $\boldsymbol{\ell_T} = \boldsymbol{\ell_T}'$. So, we conclude that $M$ is injective and thus
    \[
    |\cF(S_{\cF}(T))| \geq |\cF^{\etoe,T}(S)|.
    \]
    Now, by Lemma~\ref{lem:autoregressive-ssp} we have that $|\cB_\cF(\tree_T(x))| \leq T^{2\ARL(\cF)}$. Since every branch of $\tree_T(x)$ contains $T+1$ nodes, we conclude that
    \[
    |V_{\cF}(x,T)| \leq (T+1)|\cB_\cF(\tree_T(x))| \leq (T+1) T^{2\ARL(\cF)} \leq T^{3\ARL(\cF)}
    \]
    for all $x$, and thus $|S_{\cF}(T)| \leq m T^{3\ARL(\cF)}$. So, we get 
    \begin{equation} \label{eq:growthS-relatesto-growthS_T}
        \Gamma_{\cF^{\etoe,T}}(m) \leq \Gamma_{\cF}(m T^{3\ARL(\cF)}).
    \end{equation}
    Now, by the SSP lemma (Theorem~\ref{thm:ssp}):

    \begin{equation}\label{eq:upperbound-by-ssp}
        \Gamma_{\cF}(m T^{3\ARL(\cF)}) \leq \mleft ( mT^{3\ARL(\cF)} \mright)^{2\VC(\cF)}.
    \end{equation}
    Combining \eqref{eq:shattered-lower-bound}, \eqref{eq:growthS-relatesto-growthS_T}, and \eqref{eq:upperbound-by-ssp}, we get:
    \[
    2^m \leq (mT)^{6 \ARL(\cF) \VC(\cF)}.
    \]
    Any $m \geq 20 \ARL(\cF) \VC(\cF) \log T$ contradicts the inequality above when $T \geq 20 \ARL(\cF) \VC(\cF)$. In the complementing case, the inequality above is incorrect already for $m = \tilde{\Omega}(\ARL(\cF) \VC(\cF))$. This concludes the proof.
\end{proof}

\begin{proof}[Proof of Theorem~\ref{thm:arl-upper-bound}]
    The theorem is immediately implied from Lemma~\ref{lem:arl-upper-bound}.
\end{proof}

\subsubsection{Proof of Proposition~\ref{prop:e2e-relaxed}}
\begin{proof}[Proof of Proposition~\ref{prop:e2e-relaxed}]
    The first part of the proposition is immediate, as any candidate tree for realizing $\VC(\cF), \ARL(\cF)$ is also a Littlestone tree.

    For the ``furthermore" part, we first construct the following infinite tree $\tree$.
    We define the following natural order on the nodes: for two nodes $u,v$: $u > v$ if and only if $u$ is below $v$, or $u$ and $v$ are in the same layer, but $u$ is to the right of $v$. If the index of $v$ in this order is $i$, then we label it by $0^i$.
    Now, for every branch $B$ in this tree we define $f_B$ to be the function that agrees with the branch, and is constant $0$ on any instance outside of the branch. We define $\cF$:
    \[
    \cF := \{f_B : B\in \tree\}.
    \]
    It is immediate that $\LD(\cF) = \infty$.
    We now show that $\VC(\cF)=  \ARL(\cF)= 1$.
    For the lower bound, note that the set $\{0\}$ is shattered, and the generation tree $\tree_1(x)$ is shattered as well.
    It remains to show $\VC(\cF) \leq 1$ and $\ARL(\cF) \leq 1$.
    We begin with bounding $\VC(\cF)$.
    Let $S = \{x_1,x_2\}$ and suppose that it is shattered. Therefore, $x_1 = 0^i, x_2=0^j$ for some $i \neq j$, otherwise all functions give $0$ to both. So both $x_1,x_2$ label nodes of $\tree$. If there is no branch B such that $x_1,x_2$ label nodes in $B$, then by definition of $\tree$, any function giving $1$ to one of them must give $0$ to the other, and hence the all $1$ labeling of $S$ is not realized. Otherwise, suppose that $x_1,x_2$ lie on the same branch $B$, then the possible labelings of $S$ are:
    \[
    (f_B(x_1), f_B(x_2)), \quad \text{and} \quad (0,0)
    \]
    by definition of $\cF$. Therefore, $S$ is not shattered, and $\VC(\cF) \leq 1$.

    It remains to show $\ARL(\cF) \leq 1$. Let $x$ be an instance and let $T \geq 2$, and let $\tree$ be a perfect subtree of depth $2$ of $\tree_T(x)$. By definition of $\tree_T(x)$, every node in the right subtree of the root has $1$ in the instance labeling it, so $\cF$ is constant $0$ on all instances labeling the vertices of the right subtree of the root. Therefore, $\tree$ cannot be shattered by $\cF$, and so  $\ARL(\cF) \leq 1$.
\end{proof}

\section{Learning with Chain-of-Thought Supervision} \label{sec:cot-proofs}

The main goal of this section is to prove sample complexity upper bounds for $\ct$-learning a base class $\cF$. In contrast with $\etoe$-learning, where we have shown in Section~\ref{sec:e2e-proofs} that a dependence on the generation length $T$ is in many cases inevitable, the  bounds we prove for $\ct$-learning are given only in terms of parameters of the base class. This improves over the bounds of \cite{joshi2025theory}, that demonstrated a logarithmic dependence on $T$.

Recall that as discussed in Section~\ref{sec:prel}, fully $\ct$-learning is at least as difficult as $\ct$-learning. Therefore, to upper bound the sample complexity of $\ct$-learning, it suffices to prove sample complexity bounds for fully $\ct$-learning, which amounts to learning the class $\cF^{\ct,T}$. This is the exact approach we take in this section: we use tools from learning theory for classes with $|\cY| > 2$ and use them to learn the class $\cF^{\ct,T}$. We also establish a lower bound for fully $\ct$-learning, to characterize the limitations of this approach.

\subsection{The CoT-learning taxonomy for VC-classes} \label{sec:cot-taxonomy}

\begin{theorem} \label{thm:cot-bound-vc}
    For every class $\cF$, if $\VC(\cF) < \infty$ then for all $T$, $\cF^{\ct,T}$ is learnable with sample complexity 
    \[
    m(\epsilon, \delta) = \tilde{O} \mleft( \frac{\VC(\cF) \VC^\star(\cF) + \log (1/\delta)}{\epsilon} \mright).
    \]
    On the other hand, if $\VC(\cF) = \infty$, then for all $T$, $\cF^{\ct,T}$ is not learnable.
\end{theorem}
Theorem~\ref{thm:cot-bound-vc} states that $\VC(\cF) < \infty$ is a sufficient condition for having finite and independent of $T$ sample complexity for learning $\cF^{\ct,T}$. As in the case $\VC(\cF) < \infty$ the sample complexity bound of fully $\ct$-learning is this strong, one may wonder if  fully $\ct$-learning is useful even when $\VC(\cF) = \infty$, at least for some classes, perhaps with some sample complexity dependence on $T$. Unfortunately, the second part of Theorem~\ref{thm:cot-bound-vc} shows that this is not the case: infinite VC-dimension implies that fully $\ct$-learning is impossible.

The known bound $\VC^\star(\cF) \leq 2^{\VC(\cF) + 1}$ due to \cite{assouad1983densite} leads to an upper bound in terms of $\VC(\cF)$ only.

\begin{corollary}
    For every class $\cF$ with $\VC(\cF) < \infty$ and for all $T$, $\cF^{\ct,T}$ is learnable with sample complexity 
    \[
    m(\epsilon, \delta) = O \mleft( \frac{2^{\VC(\cF)} + \log (1/\delta)}{\epsilon} \mright).
    \]
\end{corollary}

We now prove Theorem~\ref{thm:cot-bound-vc}, starting with the upper bound.

\subsubsection{Theorem~\ref{thm:cot-bound-vc}: upper bound proof}

We will prove the upper bound of Theorem~\ref{thm:cot-bound-vc} by constructing a \emph{sample compression scheme} of size $\tilde{O}(\VC(\cF) \VC^\star(\cF))$ for $\cF^{\ct- T}$. Having that, Theorem~\ref{thm:compression-scheme} immediately implies the upper bound of Theorem~\ref{thm:cot-bound-vc}.

Our sample compression scheme is quite similar to the compression scheme of \cite{moran2016sample}, but requires a few modifications. In particular, we use an ``inflation" of a sample of size $m$ labeled by  $f^{\ct,T} \in \cF^{\ct,T}$, to a sample of size $mT$ labeled by $f \in \cF$. Let us describe our scheme in detail.

Fix a realizable sample $S = ((x_i,y_i))_{i=1}^m$ to compress, where $x_i \in \Sigma^\star$ and $y_i \in \Sigma^T$, and denote $y_i = (y_{i,1}, \ldots, y_{i,T})$. For all $i\in[m], t\in[T]$ define the binary example
\[
u_{i,t} := ((x_i \circ y_{i, \leq t-1}), y_{i,t}),
\]
and let $U(S)$ be the \emph{inflated binary-labeled} sample of $S$:
\[
U(S):= \{u_{i,t}: i\in[m], t\in[T]\}.
\]
Note that clearly, $U(S)$ is realizable by $\cF$, since $S$ is realizable by $\cF^{\ct,T}$.
Likewise, for any subset $A\subset S$ let its inflation be:
\[
U(A) :=\{u_{i,t}: x_i\in A, t\in[T]\}.
\]
Similarly, for any subset $A' \subset U(S)$, we can reconstruct the original matching subset of $S$:
\[
U^{-1}(A') = \{(x_i,y_i) \in S: \exists t, u_{i,t} \in A' \}.
\]
Note that for all $A' \subset U(S)$, we have $|U^{-1}(A')| \leq |A'| $, since the definition of $U^{-1}(A')$ implies that there is an onto function from $A'$ to  $U^{-1}(A')$.
We now define a binary hypothesis for each $A \subset S$ which is based on an ERM (\emph{empirical risk minimizer}) for $\cF$. In our context, $\ERM_{\cF}:(\Sigma^\star \times \Sigma)^\star \to \Sigma^{\Sigma^\star}$ is a function that, for any sample $A'$ labeled by a function from $\cF$, returns an arbitrary $f\in \cF$ such that $f(x) = y$ for all $(x,y) \in A'$. For any $A \subset S$ define the hypothesis
\[
h_A := \ERM_{\cF}(U(A)).
\]
Note that since $h_A$ is consistent with $U(A)$, we have $h_A^{\ct,T}$ is consistent with $A$.
We now define the family of ``base hypotheses" used for the sample compression scheme as:
\[
\cH(S) := \cH = \{h_A : A \subset S, |A| \leq s\} \subset \cF,
\]
where $s = \Theta(\VC(\cF))$ is chosen just large enough for an $1/3$-approximation for $\cF$ to exist, which is possible due to Theorem~\ref{thm:vc-approximation}.
We now show that only $\Theta(\VC^\star(\cF))$ hypotheses from $\cH$ are in fact required to obtain the correct label of every $u_{i,t} \in U(S)$ by a majority vote predictor.
\begin{lemma} \label{lem:compressed-base-hypotheses}
    There exist $n = O(\VC^\star(\cF))$ subsets $A_1, \ldots, A_n \subset S$, each of size at most $s$, such that for every $u_{i,t} \in U(S)$:
    \[
    | j \in [n]: h_{A_j}(x_i \circ y_{i, \leq t-1}) = y_{i,t}| > n/2.
    \]
\end{lemma}

\begin{proof}
    By definition of $\cH$ and Theorem~\ref{thm:vc-approximation}, for any distribution $q$ over $U(S)$, there exists $h \in \cH$ so that
    \[
    q \mleft( u_{i,t} \in U(S): h(x_i \circ y_{i, \leq t-1}) = y_{i,t} \mright) \geq 2/3.
    \]
    Indeed, as $\VC(\cH) \leq \VC(\cF)$, there is a set $A' \subset U(S)$ of size at most $s$ so that the empirical error of any base hypothesis $h \in \cH$ on $A'$ is at most $1/3$-far from its error on $q$. Since $h_{U^{-1}(A')} \in \cH$ has no errors on $A'$ and $|U^{-1}(A')| \leq |A'| \leq s$ , the above inequality is implied with $h:=h_{U^{-1}(A')}$. 
    Now, von Neumann's minimax theorem implies that there exists a distribution $p$ over $\cH$ such that for every $u_{i,t} \in U(S)$:
    \[
    p(h \in \cH: h(x_i \circ y_{i, \leq t-1}) = y_{i,t}) \geq 2/3.
    \]
    By Theorem~\ref{thm:vc-approximation} applied again, now for $\cH^\star$, there exist $A_1, \ldots, A_n \subset S$, each of size at most $s$, where $n = \Theta(\VC^\star(\cF))$, such that for all $u_{i,t} \in U(S)$:
    \begin{align*}
        \frac{|j\in [n]: h_{A_j}(x_i \circ y_{i, \leq t-1}) = y_{i,t} |}{n}
        & \geq
        p(h \in \cH: h(x_i \circ y_{i, \leq t-1}) = y_{i,t}) - 1/8 \\
        & \geq
        2/3 - 1/8 \\
        &> 0.54.
    \end{align*}
    This concludes the proof.
\end{proof}

We may now prove the following key lemma.

\begin{lemma} \label{lem:compression}
    There is a sample compression scheme for $\cF^{\ct,T}$ of size at most $\tilde{O}(\VC(\cF) \VC(\cF^\star))$.
\end{lemma}

\begin{proof}
    The compression map for a sample $S$ is given by
    \[
    \kappa(S) = \mleft( \bigcup \{A_j\}_{j=1}^n, i \mright)
    \]
    where $\{A_j\}_{i=1}^n$ are the subsets in the statement of Lemma~\ref{lem:compressed-base-hypotheses}, and $i \in I$ is a side information enabling to deduce $A_1, \ldots, A_n$ from the set $\bigcup \{A_j\}_{j=1}^n$. As explained in \cite{moran2016sample}, there exists $I$ with $\log |I| = \tilde{O}(s n)$. Therefore, the size of the compression scheme is $\tilde{O}(s n) = \tilde{O}(\VC(\cF) \VC^\star(\cF))$.

    It remains to explain the reconstruction of the labels of $S$ from $\mleft( \bigcup \{A_j\}_{j=1}^n, i \mright)$. We first use the side information to find $A_1, \ldots, A_n$. Now, define the function $h_{\Maj}$ as
    \[
    h_{\Maj}(x) = \Maj(h_{A_1}(x), \ldots,  h_{A_n}(x))
    \]
    where for any $v \in \{0,1\}^\star$, $\Maj(v)$ returns the symbol appearing in at least $n/2$ indices (break ties arbitrarily).
    By Lemma~\ref{lem:compressed-base-hypotheses}, $h_{\Maj}$ returns the correct label for every $u_{i,t} \in U(S)$. Thus, for any $(x_i, y_i) \in S$, we can reconstruct $y_i$ by reconstructing $\{u_{i,t}\}_{t \in [T]}$. So, we define the reconstruction map for $\mleft( \bigcup \{A_j\}_{j=1}^n, i \mright)$ as
    \[
    \rho\mleft( \bigcup \{A_j\}_{j=1}^n, i \mright) = h_\Maj^{\ct,T}.
    \]
    Since $h_{\Maj}$ returns the correct label for every $u_{i,t} \in U(S)$, $h_\Maj^{\ct,T}$ returns the correct label for every $x_i \in S$, which is the desired property of a sample compression scheme.
\end{proof}

\subsubsection{Theorem \ref{thm:cot-bound-vc}: Lower bound proof}

\begin{theorem} \label{thm:infty-vc-imnplies-infty-ds}
    For every base class $\cF$, if $\VC(\cF) = \infty$ then $\cF^{\ct,T}$ is not learnable for all $T$.
\end{theorem}

In order to prove this theorem, we will prove a stronger result from which Theorem~\ref{thm:infty-vc-imnplies-infty-ds} is immediately implied. For any function $f$ with range contained in $\{0,1\}^\star$, define the function $f[1]$ with the same domain as $f$ and range $\{0,1\}$ as follows. For every $x$ in the domain, let
\[
f[1](x) = (f(x))_1.
\]
That is, $f[1](x)$ is the first bit of $f(x)$. Likewise, for a function class $\cF$ over some domain $\cX$ and label set contained in $\{0,1\}^\star$, define
\[
\cF[1] := \{f[1] : f \in \cF\}.
\]
We prove the following
\begin{theorem} \label{thm:finite-ds-imply-finite-vc}
    Let $T \in \mathbb{N}$, and let $\cF_T: \cX \to \{0,1\}^T$ be a PAC-learnable function class. Then, $\cF_T[1]$ is also PAC-learnable.
\end{theorem}

\begin{proof}
    Since $\cF_T$ is PAC-learnable, it has a finite Natarajan dimension \cite{natarajan1989learning}, which we denote here by $d$.
    First, observe that for any $m \in \mathbb{N}$ and every sequence $\boldsymbol{x} \subset \cX^m$ we have 
    \begin{equation}
        |\cF_T[1](\boldsymbol{x})| \leq |\cF_T(\boldsymbol{x})|,
    \end{equation}
    since there is an onto mapping from $\cF_T(\boldsymbol{x})$ to $\cF_T[1](\boldsymbol{x})$: map a labeling from $\cF_T(\boldsymbol{x})$ to its first-bit projection, which is a labeling from $\cF_T[1](\boldsymbol{x})$. This mapping is onto by definition of $\cF_T[1]$.
    By definition of the growth function, this implies for every $m$:
    \begin{equation} \label{eq:growth-f-and-f-t}
        \Gamma_{\cF_T[1]}(m) \leq \Gamma_{\cF_T}(m).
    \end{equation}
    Now, let $S \subset \cX$ be a set shattered by $\cF_T[1]$ of size $m$.
    We will now lower bound $\Gamma_{\cF_T[1]}(m)$ and upper bound $\Gamma_{\cF}(m)$, which will imply a useful inequality.
    First, we have
    \begin{equation} \label{eq:growth-1-lower-bound}
        \Gamma_{\cF_T[1]}(m) = 2^m,
    \end{equation}
    since $S$ is of size $m$ and shattered by $\cF_T[1]$.
    Second, by the generalization of the SSP Lemma for multiclass learning (Theorem~\ref{thm:ssp}) we have:
    \begin{equation} \label{eq:growth-t-upper-bound}
        \Gamma_{\cF_T}(m) \leq (e m 2^T)^{2d}.
    \end{equation}
    Combining \eqref{eq:growth-f-and-f-t}, \eqref{eq:growth-1-lower-bound} and \eqref{eq:growth-t-upper-bound}, we obtain:
    \[
        2^m \leq (e m 2^T)^{2d}.
    \]
    Solving this for $m$ gives $m = \tilde{O}(T d)$.
\end{proof}

\begin{proof}[Proof of Theorem~\ref{thm:infty-vc-imnplies-infty-ds}]
    We apply Theorem~\ref{thm:finite-ds-imply-finite-vc} with $\cF_T = \cF^{\ct,T}$ and $\cF_T[1] = \cF$. So, learnability of $\cF^{\ct,T}$ implies learnability of $\cF$. By the contrapositive claim, if $\VC(\cF) = \infty$ then $\cF^{\ct,T}$ is not learnable.  
\end{proof}

\subsubsection{Proof of Theorem~\ref{thm:cot-bound-vc}}

\begin{proof}[Proof of Theorem~\ref{thm:cot-bound-vc}]
    The lower bound is given by Theorem~\ref{thm:infty-vc-imnplies-infty-ds}. The upper bound is implied as follows.
    By Lemma~\ref{lem:compression}, we have a sample compression scheme for $\cF^{\ct,T}$, of size $\tilde{O}(\VC(\cF) \VC^\star (\cF))$. Theorem~\ref{thm:compression-scheme} now implies the upper bound.
\end{proof}

\subsection{Upper bound for classes with stable sample compression schemes} \label{sec:cot-linear-classifiers}

In this section, we prove that for base classes with a stable sample compression scheme of size $k$, we can prove a sample complexity bound proportional to $k$ for fully $\ct$-learning, that is, for learning the class $\cF^{\ct,T}$. A main application of this approach is an improved bound for the base class of linear classifiers.  

Following \cite{joshi2025theory}, we define the class of linear classifiers in dimension $d$ over the domain $\Sigma^\star$ as follows:
\[
\cF_{d, \lin} := \{f_{\boldsymbol{w}, b}: \boldsymbol{w}\in \mathbb{R}^d, b\in \mathbb{R}\}
\]
where the function $f_{\boldsymbol{w}, b}$ is given as
\[
f_{\boldsymbol{w}, b}(x) = 1 \mleft[ \sum_{i=1}^{\min\{d, |\boldsymbol{x}|\}} \boldsymbol{w}[-i] x[-i] + b \geq 0 \mright]
\]
for all $x \in \Sigma^\star$. Note that we cannot use the standard definition of linear classifiers, since we must assume that the domain is $\Sigma^\star$ to make the autoregressive chain-of-thought generation feasible. Therefore, the function $f_{\boldsymbol{w}, b}$ takes into account only the last $d$ bits of the input $x \in \Sigma^\star$.

\begin{theorem} \label{thm:linear-classifiers-cot-bound}
    For all $d,T \in \mathbb{N}$, $\cF_{d, \lin}^{\ct,T}$ is learnable with sample complexity
    \[
    m(\epsilon, \delta) = O \mleft( \frac{d + \log (1/\delta)}{\eps} \mright).
    \]
\end{theorem}
This improves over the previous bound $m(\epsilon, \delta) = O \mleft( \frac{d^2 + \log (1/\delta)}{\eps} \mright)$ of \cite{joshi2025theory}.
As in the proof of Theorem~\ref{thm:cot-bound-vc}, sample compression schemes are useful also for proving Theorem~\ref{thm:linear-classifiers-cot-bound}. The main result we use is the following.
\begin{theorem} \label{thm:linear-stable-sample-compression}
    If $\cF$ has a stable sample compression scheme of size $k$, then for all $T$,  $\cF^{\ct,T}$ has a stable sample compression scheme of size $k$.
\end{theorem}

\begin{proof}
    Let $(\kappa, \rho)$ be the compression and reconstruction maps for $\cF$. For every $T$, we define mappings $(\kappa^T, \rho^T)$ for $\cF^{\ct,T}$ as follows. Fix a sample $S = ((x_i,y_i))_{i=1}^m$ for $\cF^{\ct,T}$. That is, $x_i \in \Sigma^\star, y_i \in \Sigma^T$ where for some $f \in \cF$: $f^{\ct,T}(x_i) = y_i$ for all $i$. Define
    \[
    \kappa^T(S) := U^{-1}(\kappa(U(S))).
    \]
    In words, we first inflate $S$ to $U(S)$, which is a binary sample for $\cF$. Then we can use $\kappa$ and obtain a compressed binary sample for $\cF$. Then, we construct from this compressed sample a compressed sample of $S$, by the definition of $U^{-1}$. For the reconstruction map, fix $A \subset S$, and define:
    \[
    \rho^T(A) := (\rho(\kappa(U(A))))^{\ct,T}.
    \]
    In words, for reconstruction, we first inflate $A$ to $U(A)$, then compress $U(A)$ using the original compression $\kappa$, and then use the original reconstruction $\rho$. Then, we just use the standard $\ct,T$ generator for the next-token-generator obtained from the original reconstructor $\rho$. Now, we need to prove that $(\kappa^T, \rho^T)$ is indeed a sample compression scheme for $\cF^{\ct,T}$, that is stable, and that its size is at most $k$.
    \paragraph{Correctness.}
    First, we show that $(\kappa^T, \rho^T)$ is indeed a sample compression scheme. Let $A := \kappa^T(S)$. We need to show that the function $\rho^T(A)$ is consistent with $S$. By definition of $A$, we have $A \subset S$ and $\kappa(U(S)) \subset U(A)$. Therefore, we have
    \[
    \kappa(U(S)) \subset U(A) \subset U(S).
    \]
    By stability of $(\kappa, \rho)$, the above inclusion relation implies $\kappa(U(A)) = \kappa(U(S))$. Therefore, we have
    \[
    \rho^T(A) = (\rho(\kappa(U(S))))^{\ct,T}.
    \]
    Since $(\kappa, \rho)$ is a sample compression scheme, we have
    \[
    (\rho(\kappa(U(S))))(x_i \circ y_{i, \leq t-1}) = y_{i,t}
    \]
    for all $i \in [m], t \in [T]$. So by definition:
    \[
    \rho^T(A)(x_i) = y_i
    \]
    for all $i \in [m]$.

    \paragraph{Stability.}
    Let $A$ such that $\kappa^T(S) \subset A \subset S$. Then
    \[
    \kappa(U(S)) \subset U(A) \subset U(S).
    \]
    Therefore, stability of $(\kappa, \rho)$ gives $\kappa(U(S)) = \kappa(U(A))$, which implies $\kappa^T(A) = \kappa^T(S)$, as required.

    \paragraph{Size.}
    We have
    \[
    \lvert \kappa^T(S) \rvert
    =
    \lvert U^{-1}(\kappa(U(S))) \rvert
    \leq
    \lvert \kappa(U(S)) \rvert
    \leq
    k,
    \]
    which concludes the proof.
\end{proof}

We may now prove Theorem~\ref{thm:linear-classifiers-cot-bound}.

\begin{proof}[Proof of Theorem~\ref{thm:linear-classifiers-cot-bound}]
    It was proved by \cite{vapnik1971uniform, long2020complexity} that the standard class of linear classifiers in $\mathbb{R}^d$ has a stable sample compression scheme of size $d+1$ (and no side information). As $\cF_{d, \lin}$ is essentially the same class defined only on the boolean cube in $\mathbb{R}^d$, the same result holds for $\cF_{d, \lin}$. Therefore, Theorem~\ref{thm:linear-stable-sample-compression} implies that $\cF_{d, \lin}^{\ct,T}$ has a stable sample compression scheme of size $d+1$ for all $T$. Theorem~\ref{thm:stable-compression-scheme} now implies the stated bound.
\end{proof}

\section{Non-VC classes could have no uniform bound} \label{sec:non-vc-proofs}
In Section~\ref{sec:cot-proofs}, we established that the technique of bounding the sample complexity of $\cF^{\ct- T}$  to obtain bounds for $\ct$-learning is useful if and only if $\VC(\cF) < \infty$ (Theorem~\ref{thm:cot-bound-vc}). However, this result does not imply that there is no other way to learn $\cF^{\etoe,T}$ with $\ct$ supervision, since this is an easier task than learning  $\cF^{\ct,T}$, which includes learning the entire chain-of-thought. Indeed, when learning $\cF^{\etoe,T}$ with $\ct$ supervision, the learner receives the entire chain-of-thought during training, as in learning $\cF^{\ct,T}$, but tested only on the final bit of the chain-of-thought, exactly as tested when learning $\cF^{\etoe,T}$ without $\ct$ supervision. So, while the focus on VC-classes in classic PAC-learning is well justified by VC-theory, the situation for non VC-classes in the autoregressive learning model is not as clear.

In this section, we provide some partial justification for focusing on VC-classes. We show that a general characterization of the sample complexity of learning $\cF^{\etoe,T}$ when $\VC(\cF) = \infty$ cannot exist in a very strong sense. Concretely, we prove that there exists a class $\cF$ (with infinite VC-dimension), such that for all even $T$ values it holds that $\VC(\cF^{\etoe,T}) = 0$, while for all odd $T$ values, $\cF^{\etoe,T}$ is not learnable even with CoT supervision. That is, for even $T$ values, $\cF^{\etoe,T}$ is trivially learnable (not even a single data point is required to learn it with perfect precision). In stark contrast, for odd values of $T$, $\cF^{\etoe,T}$ is not learnable even if the entire chain-of-thought is revealed to the learner during training. Formally, we prove the following

\begin{theorem} \label{thm:infinite-vc-no-characterization}
    There exists a class $\cF$ such that for all even $T$ we have $\VC(\cF^{\etoe,T}) = 0$, and for all odd $T$, $\cF$ is not $\ct$-learnable.
\end{theorem}

While strong, Theorem~\ref{thm:infinite-vc-no-characterization} provides only a partial justification for focusing on VC-classes, as there might be some other dimension that captures learnability of $\cF^{\etoe,T}$ (with or without $\ct$ supervision) for all $T$. Determining whether such a dimension exists remains open.

Let us define the class $\cF$ used in the proof of Theorem~\ref{thm:infinite-vc-no-characterization}. For every infinite string $b = b_0,b_1,.... \in \{0,1\}^\mathbb{N}$, define the function $f_b$ as follows.
For any $k \geq 1, y \in \{0,1\}$ and $t \geq 0$, denote:
\[
Q_k := 0^k1, \quad A_{k,y,t} := 0^k1 (y0)^t y, \quad B_{k,y,t} := 0^k 1 (y 0 )^{t+1}.
\]
Now, define
\[
f_b(x) :=
\begin{cases}
    b_k & x=Q_k \text{ for some } k, \\
    0   & x = A_{k,y,t} \text{ for some } k,y,t, \\
    y   & x = B_{k,y,t} \text{ for some } k,y,t, \\
    0   & \text{Otherwise.}
\end{cases}
\]
Let
\[
\cF := \{f_b: b \in \{0,1\}^\mathbb{N}\}.
\]

Now, let us prove the part of Theorem~\ref{thm:infinite-vc-no-characterization} about even $T$ values.
\begin{lemma}\label{lem:infinite-vc-no-characterization-even}
    We have $\VC(\cF^{\etoe,T}) = 0$ for all even $T$.
\end{lemma}

\begin{proof}
    Let $x \in \Sigma^\star$. We prove that $\{x\}$ is not shattered, by showing that $f_b^{\etoe,T}(x) = f^{\etoe,T}_{b'}(x)$ for all $b,b'$. Let us verify that the claim holds in every possible case. If $x = Q_k$ for some $k$, then
    \[
    f_b^{\ct,T}(x) = (b_k0)^{\frac{T}{2}}, \quad \text{and} \quad f_{b'}^{\ct,T}(x) = (b'_k0)^{\frac{T}{2}}.
    \]
    Therefore, $f_b^{\etoe,T}(x) = f_{b'}^{\etoe,T}(x) = 0$. For any other $x \in \Sigma^\star$, the definition of $f_b(x)$ is independent of $b$ and thus $f_b^{\ct,T}(x) = f_{b'}^{\ct,T}(x)$, which immediately implies $f_b^{\etoe,T}(x) = f_{b'}^{\etoe,T}(x)$.
\end{proof}

Now, we prove the part of Theorem~\ref{thm:infinite-vc-no-characterization} about odd $T$ values. The proof is very similar to the classic lower bound proof of the fundamental theorem of PAC learning, which establishes that classes of infinite VC-dimension are not learnable. In the classic lower bound, the proof relies on the idea that learning a part of a shattered set does not expose any information on the other parts of it. The same idea applies here: If the target base function is $f_b$, then learning a part of $b$ does not expose any information on the other parts of it, since all $b\in \{0,1\}^{\mathbb{N}}$ are possible.

\begin{lemma}\label{lem:infinite-vc-no-characterization-odd}
    For all odd $T$, $\cF$ is not $\ct$-learnable.
\end{lemma}

\begin{proof}
    Fix an arbitrary learner $\Lrn$ for $\ct$-learning $\cF$ with generation length $T$, where $T$ is odd. Let $n \in \mathbb{N}$ be the sample size used by $\Lrn$. We will show that there exists a realizable distribution $D$ over $\cX$ and a target function $f_b \in \cF$ such that
    \[
    \Pr_{x_1, \ldots, x_n \sim D^n} \mleft[\Pr_{x \sim D}[f_{\Lrn}(x) \neq f_b^{\etoe,T}(x)] \geq 1/4 \mright] \geq 1/2,
    \]
    which implies that $\cF$ is not $\ct$-learnable with generation length $T$. Let $m:= 2n$, and let $D$ be the uniform distribution over $\{Q_1, \ldots, Q_m\}$. We choose the target concept at random: draw $b_1, \ldots, b_m \sim \Ber(1/2)$ and let the target concept be $f_b$ where $b := b_1 \ldots b_m 0^{\mathbb{N}}$.

    Fix a realized training sample $S = (x_1, f_b^{\ct -T}(x_1)), \ldots, (x_n, f_b^{\ct -T}(x_n))$  and let $I:= \{k \in [m]: Q_k \in S\}$ be the set of indices $k$ where $Q_k$ appears in $S$. Note that $|I| \leq n= m/2$, and thus the set $J := [m] \backslash I$ is of size at least $n = m/2$. The crucial fact is that since $f_b^{\etoe,T}(Q_k) = b_k$ and $b_k \sim \Ber(1/2)$ iid, for all $k \in J$ we have:
    \[
    \Pr[f_b^{\etoe,T}(Q_k) = 0|S] = \Pr[b_k = 0|S] = 1/2.
    \]
    That is, $f_b^{\etoe,T}(Q_k) = 0$ with probability exactly half, even conditioned on $S$.
    Therefore, conditioned on $S$ and any internal randomness of the learner we have
    \[
    \Pr[f_{\Lrn}(Q_k) \neq f_b^{\etoe,T}(Q_k) | S, \Lrn] = 1/2
    \]
    for all $k \in J$. Now let $B$ be the number of ``bad" $k$ values, that is,
    \[
    B:= \lvert \{k \in J: f_{\Lrn}(Q_k) \neq f_b^{\etoe,T}(Q_k) \} \rvert.
    \]
    So, $B| S, \Lrn \sim \Bin (|J|, 1/2)$. Since the binomial distribution is symmetric, we have:
    \begin{equation} \label{eq:infinite-vc-no-characterization-odd-bad-bound}
        \Pr[B \geq |J|/2 | S, \Lrn] \geq 1/2.
    \end{equation}
    Now, since $D$ is uniform over $x_1, \ldots, x_m$, we have
    \[
    \Pr_{x \sim D}[f_{\Lrn}(x) \neq f_b^{\etoe,T}(x)] = B/m.
    \]
    So, we have:
    \[
    \Pr_{x \sim D} \mleft[f_{\Lrn}(x) \neq f_b^{\etoe,T}(x) | B \geq |J|/2  \mright] \geq \frac{|J|/2}{m} \geq \frac{n/4}{m} = 1/4.
    \]
    From \eqref{eq:infinite-vc-no-characterization-odd-bad-bound}, the event $B \geq |J|/2$ occurs with probability at least $1/2$, conditioned on every $S$ and $\Lrn$, and so applying the law of total probability concludes the proof.
\end{proof}

\begin{proof}[Proof of Theorem~\ref{thm:infinite-vc-no-characterization}]
    The theorem immediately follows from Lemma~\ref{lem:infinite-vc-no-characterization-even} and Lemma~\ref{lem:infinite-vc-no-characterization-odd}.
\end{proof}

\section{Acknowledgments}
We thank Gal Vardi for presenting the time-invariant autoregressive learning model to us, and for helpful discussions and explanations on the work \cite{joshi2025theory}.
We also thank Zachary Chase, Amit Daniely, Nathan Srebro and Elchanan Mossel for insightful discussions on the problems in the scope of this work.

Steve Hanneke acknowledges support by grant no.\ 2024243 from the United States - Israel Binational Science Foundation (BSF).

Idan Mehalel is supported by the European Research Council (ERC) under the European Union’s Horizon 2022 research and innovation program (grant agreement No. 101041711), the Israel Science Foundation (grant number 2258/19), and the Simons Foundation (as part of the Collaboration on the Mathematical and Scientific Foundations of Deep Learning).

Shay Moran is a Robert J.\ Shillman Fellow; he acknowledges support by Israel PBC-VATAT, by the Technion Center for Machine Learning and Intelligent Systems (MLIS), and by the the European Union (ERC, GENERALIZATION, 101039692). Views and opinions expressed are however those of the author(s) only and do not necessarily reflect those of the European Union or the European Research Council Executive Agency. Neither the European Union nor the granting authority can be held responsible for them.

\bibliographystyle{alphaurl}
\bibliography{bibShay.bib}

@book{vapnik:74,
author = {V. Vapnik and A. Chervonenkis},
title = {Theory of Pattern Recognition},
publisher = {Nauka, Moscow},
year = 1974
}

@book{shalev2014understanding,
  title={Understanding machine learning: From theory to algorithms},
  author={Shalev-Shwartz, Shai and Ben-David, Shai},
  year={2014},
  publisher={Cambridge university press}
}

@article{vapnik1971uniform,
  title={On the uniform convergence of relative frequencies of events to their probabilities},
  author={Vapnik, Vladimir N. and Chervonenkis, Alexey Ya.},
  journal={Theory of Probability \& Its Applications},
  volume={16},
  number={2},
  pages={264--280},
  year={1971},
  publisher={SIAM}
}

@article{joshi2025theory,
  title={A Theory of Learning with Autoregressive Chain of Thought},
  author={Joshi, Nirmit and Vardi, Gal and Block, Adam and Goel, Surbhi and Li, Zhiyuan and Misiakiewicz, Theodor and Srebro, Nathan},
  journal={arXiv preprint arXiv:2503.07932},
  year={2025}
}

@article{walker2022tree,
  title={Tree dimension and the Sauer-Shelah dichotomy},
  author={Walker, Roland},
  journal={arXiv preprint arXiv:2203.12211},
  year={2022}
}

@inproceedings{assouad1983densite,
  title={Densit{\'e} et dimension},
  author={Assouad, Patrick},
  booktitle={Annales de l'institut Fourier},
  volume={33},
  number={3},
  pages={233--282},
  year={1983}
}

@article{littlestone1986relating,
  title={Relating data compression and learnability},
  author={Littlestone, Nick and Warmuth, Manfred},
  year={1986},
  publisher={Citeseer}
}

@article{moran2016sample,
  title={Sample compression schemes for VC classes},
  author={Moran, Shay and Yehudayoff, Amir},
  journal={Journal of the ACM (JACM)},
  volume={63},
  number={3},
  pages={1--10},
  year={2016},
  publisher={ACM New York, NY, USA}
}

@inproceedings{bousquet2020proper,
  title        = {Proper Learning, Helly Number, and an Optimal SVM Bound},
  author       = {Bousquet, Olivier and Hanneke, Steve and Moran, Shay and Zhivotovskiy, Nikita},
  booktitle    = {Proceedings of the Thirty Third Conference on Learning Theory},
  series       = {Proceedings of Machine Learning Research},
  volume       = {125},
  pages        = {582--609},
  year         = {2020},
  editor       = {Abernethy, Jacob and Agarwal, Shivani},
  publisher    = {PMLR},
  note         = {COLT 2020 (9--12 July 2020)},
  url          = {https://proceedings.mlr.press/v125/bousquet20a.html},
  eprint       = {2005.11818},
  archivePrefix= {arXiv},
  primaryClass = {cs.LG}
}

@article{haussler1995generalization,
  title={A generalization of Sauer's lemma},
  author={Haussler, David and Long, Philip M},
  journal={Journal of Combinatorial Theory, Series A},
  volume={71},
  number={2},
  pages={219--240},
  year={1995},
  publisher={Elsevier}
}

@article{natarajan1989learning,
  title={On learning sets and functions},
  author={Natarajan, Balas K},
  journal={Machine Learning},
  volume={4},
  number={1},
  pages={67--97},
  year={1989},
  publisher={Springer}
}

@article{valiant1984theory,
  title={A theory of the learnable},
  author={Valiant, Leslie G},
  journal={Communications of the ACM},
  volume={27},
  number={11},
  pages={1134--1142},
  year={1984},
  publisher={ACM New York, NY, USA}
}

@article{blumer1989learnability,
  title={Learnability and the Vapnik-Chervonenkis dimension},
  author={Blumer, Anselm and Ehrenfeucht, Andrzej and Haussler, David and Warmuth, Manfred K},
  journal={Journal of the ACM (JACM)},
  volume={36},
  number={4},
  pages={929--965},
  year={1989},
  publisher={ACM New York, NY, USA}
}

@article{ehrenfeucht1989general,
  title={A general lower bound on the number of examples needed for learning},
  author={Ehrenfeucht, Andrzej and Haussler, David and Kearns, Michael and Valiant, Leslie},
  journal={Information and Computation},
  volume={82},
  number={3},
  pages={247--261},
  year={1989},
  publisher={Elsevier}
}

@article{hanneke2016optimal,
  title={The optimal sample complexity of PAC learning},
  author={Hanneke, Steve},
  journal={Journal of Machine Learning Research},
  volume={17},
  number={38},
  pages={1--15},
  year={2016}
}

@article{sauer1972density,
  title={On the density of families of sets},
  author={Sauer, Norbert},
  journal={Journal of Combinatorial Theory, Series A},
  volume={13},
  number={1},
  pages={145--147},
  year={1972},
  publisher={Elsevier}
}

@article{shelah1972combinatorial,
  title={A combinatorial problem; stability and order for models and theories in infinitary languages},
  author={Shelah, Saharon},
  journal={Pacific Journal of Mathematics},
  volume={41},
  number={1},
  pages={247--261},
  year={1972},
  publisher={Mathematical Sciences Publishers}
}

@inproceedings{long2020complexity,
  title={On the complexity of proper distribution-free learning of linear classifiers},
  author={Long, Philip M and Long, Raphael J},
  booktitle={Algorithmic Learning Theory},
  pages={583--591},
  year={2020},
  organization={PMLR}
}

@article{malach2023auto,
  title={Auto-regressive next-token predictors are universal learners},
  author={Malach, Eran},
  journal={arXiv preprint arXiv:2309.06979},
  year={2023}
}

@article{huang2025transformers,
  title={Transformers provably learn chain-of-thought reasoning with length generalization},
  author={Huang, Yu and Wen, Zixin and Singh, Aarti and Chi, Yuejie and Chen, Yuxin},
  journal={arXiv preprint arXiv:2511.07378},
  year={2025}
}

@article{tsilivis2025reinforcement,
  title={How reinforcement learning after next-token prediction facilitates learning},
  author={Tsilivis, Nikolaos and Malach, Eran and Ullrich, Karen and Kempe, Julia},
  journal={arXiv preprint arXiv:2510.11495},
  year={2025}
}

@article{shalev2025reasoning,
  title={From Reasoning to Super-Intelligence: A Search-Theoretic Perspective},
  author={Shalev-Shwartz, Shai and Shashua, Amnon},
  journal={arXiv preprint arXiv:2507.15865},
  year={2025}
}

@article{geneson2025mistake,
  title={Mistake-bounded online learning with operation caps},
  author={Geneson, Jesse and Li, Meien and Tang, Linus},
  journal={arXiv preprint arXiv:2509.03892},
  year={2025}
}

@article{daniely2025online,
  title={Online Learning of Neural Networks},
  author={Daniely, Amit and Mehalel, Idan and Mossel, Elchanan},
  journal={arXiv preprint arXiv:2505.09167},
  year={2025}
}

@article{balcan2025learning,
  title={On Learning Verifiers for Chain-of-Thought Reasoning},
  author={Balcan, Maria-Florina and Blum, Avrim and Li, Zhiyuan and Sharma, Dravyansh},
  journal={arXiv preprint arXiv:2505.22650},
  year={2025}
}

@article{altabaa2025cot,
  title={CoT Information: Improved Sample Complexity under Chain-of-Thought Supervision},
  author={Altabaa, Awni and Montasser, Omar and Lafferty, John},
  journal={arXiv preprint arXiv:2505.15927},
  year={2025}
}

@inproceedings{LechnerB24,
  author       = {Tosca Lechner and
                  Shai Ben{-}David},
  editor       = {Shipra Agrawal and
                  Aaron Roth},
  title        = {Inherent limitations of dimensions for characterizing learnability
                  of distribution classes},
  booktitle    = {The Thirty Seventh Annual Conference on Learning Theory, June 30 -
                  July 3, 2023, Edmonton, Canada},
  series       = {Proceedings of Machine Learning Research},
  pages        = {3353--3374},
  publisher    = {{PMLR}},
  year         = {2024},
  url          = {https://proceedings.mlr.press/v247/lechner24a.html},
  bibsource    = {dblp computer science bibliography, https://dblp.org}
}

@article{doliwa2014recursive,
  title={Recursive teaching dimension, VC-dimension and sample compression},
  author={Doliwa, Thorsten and Fan, Gaojian and Simon, Hans Ulrich and Zilles, Sandra},
  journal={The Journal of Machine Learning Research},
  volume={15},
  number={1},
  pages={3107--3131},
  year={2014},
  publisher={JMLR. org}
}

@article{talagrand1994sharper,
  title={Sharper bounds for Gaussian and empirical processes},
  author={Talagrand, Michel},
  journal={The Annals of Probability},
  pages={28--76},
  year={1994},
  publisher={JSTOR}
}

\end{document}